\documentclass[preprint, 11pt, a4paper]{article} 

\usepackage[truedimen,margin=25truemm]{geometry}

\usepackage{times}
\usepackage{graphicx} 
\usepackage{epsfig} 
\usepackage{subfigure} 

\usepackage{natbib}
\usepackage{hyperref}
\usepackage{algorithm}
\usepackage{algpseudocode} 
\usepackage{placeins}

\usepackage{url}

\usepackage{graphicx} 
\usepackage{epsfig}
\usepackage{subfigure}
\usepackage{enumitem}
\usepackage{amssymb}
\usepackage{amsmath}
\usepackage{amsthm}
\usepackage{mathtools}
\usepackage[raggedright]{titlesec} 
\usepackage{bm}
\usepackage{mathrsfs}
\theoremstyle{plain}
\newtheorem{theorem}{Theorem}[section]
\newtheorem{proposition}[theorem]{Proposition}
\newtheorem{lemma}[theorem]{Lemma}
\newtheorem{corollary}[theorem]{Corollary}

\theoremstyle{definition}
\newtheorem{definition}[theorem]{Definition}

\theoremstyle{remark}
\newtheorem{remark}[theorem]{Remark}

\usepackage{arydshln}
\DeclareMathAlphabet\mathbfcal{OMS}{cmsy}{b}{d}

\DeclareMathOperator*{\argmin}{argmin}
\makeatletter
\def\hlinewd#1{%
  \noalign{\ifnum0=`}\fi\hrule \@height #1 \futurelet
   \reserved@a\@xhline}
\makeatother

\title{
	Guaranteed Low-Multilinear-Rank Tensor Recovery from Modewise Measurements via
	Normalized Block-Weighted Riemannian Gradient Descent
	\thanks{This work was supported in part by the National Key Research and
		Development Program of China under Grant 2023YFA1008502; in part by the
		Fundamental Research Funds for the Central Universities under Grant
		SWU-KR25013; and in part by the National Natural Science Foundation of
		China under Grant 12101512.}
}

\author{
	Yushi Zhou%
	\thanks{School of Mathematics and Statistics, Southwest University,
		Chongqing, China. Email: \texttt{zhou2003625@163.com}.}
	\qquad
	Feng Zhang%
	\thanks{Corresponding author. School of Mathematics and Statistics,
		Southwest University, Chongqing, China.
		Email: \texttt{zfmath@swu.edu.cn}.}
}

\begin{document}

\maketitle

\begin{abstract}
	We consider the recovery of low-multilinear-rank tensors from linear measurements and propose an adaptive block-weighted modewise Riemannian gradient descent method. The method combines memory-efficient modewise measurements with a normalized adaptive weighting strategy for the core and factor components of the Riemannian gradient. The weighting improves convergence without increasing the multilinear-rank bound of the search direction or the size of the reduced core used for retraction. Under the tensor restricted isometry property and a suitable initialization, we establish local linear convergence and derive sampling guarantees for sub-Gaussian and subsampled orthogonal with random sign (SORS) measurements. Numerical experiments on synthetic low-Tucker-rank tensors show that the proposed method reduces iteration counts and computational time while maintaining reliable recovery performance, especially near the recovery threshold and for structured SORS measurements.
\end{abstract}

\section{Introduction}
\label{sec:intro}

The tensor recovery problem arises in a wide range of applications,
including computer vision \cite{li2010tensor,liu2012tensor},
image inpainting \cite{zhang2016exact,gilman2022grassmannian},
multi-temporal image reconstruction \cite{lin2026robust},
machine learning \cite{anandkumar2014tensor},
signal processing \cite{cichocki2015tensor,zhang2021learning},
and quantum state tomography \cite{gross2010quantum,kueng2017low}.
In this paper, we consider the recovery of an unknown tensor
$\mathcal X\in\mathbb R^{n_1\times\cdots\times n_d}$ that admits a Tucker
decomposition \cite{tucker1966some} of multilinear rank
$\boldsymbol r=(r_1,\ldots,r_d)$, from a limited number of linear measurements
\begin{equation}
	\boldsymbol y=\bm{\mathscr L}(\mathcal X),
	\label{eq:intro_measurement_model}
\end{equation}
where
$\bm{\mathscr L}:\mathbb R^{n_1\times\cdots\times n_d}\rightarrow
\mathbb R^M$ is a linear measurement operator and
$M\ll\prod_{i=1}^{d}n_i$.
Since the measurement dimension is considerably smaller than the
ambient tensor dimension, the recovery problem is generally
underdetermined without additional structural assumptions. By imposing
the low-multilinear-rank prior, tensor recovery can be formulated as the
constrained least-squares problem
\begin{equation}
	\min_{\mathcal X\in\mathbb R^{n_1\times\cdots\times n_d}}
	\frac{1}{2}
	\left\|
	\bm{\mathscr L}(\mathcal X)-\boldsymbol y
	\right\|_2^2
	\quad
	\text{subject to}\quad
	\operatorname{mulrank}(\mathcal X)=\boldsymbol r.
	\label{eq:intro_recovery_problem}
\end{equation}

A conventional realization of the measurement operator $\bm{\mathscr L}$ in
\eqref{eq:intro_measurement_model} first vectorizes the tensor and then
applies a single dense measurement matrix of size $M\times\prod_{i=1}^{d}n_i$. This
vectorization-based construction underlies many classical algorithms for
solving the recovery problem \eqref{eq:intro_recovery_problem} and its
sparse counterpart, such as $\ell_1$-minimization \cite{foucart2011sparse,mo2011new},
CoSaMP \cite{needell2009cosamp,foucart2011sparse}, and iterative hard
thresholding \cite{blumensath2008iterative}. Its practical drawback,
however, lies in the size of this matrix: for high-order or large-scale
tensors, merely generating and storing this matrix can require more memory than storing
the target tensor itself. Consequently, although dense vectorized
measurements provide a convenient theoretical model, their storage cost
can render them impractical for large-scale tensor recovery.

Modewise measurement operators provide a structured and
memory-efficient alternative to conventional vectorization-based
measurements. Following the modewise measurement framework of \cite{iwen2021lower,jin2021faster}, as further developed in \cite{haselby2023modewise}, a two-stage modewise linear operator $\bm{\mathscr L}:\mathbb R^{n_1\times\cdots\times n_d}\longrightarrow\mathbb R^{m_1\times\cdots\times m_{d'}}$ can be expressed as
\begin{equation}
	\bm{\mathscr L}(\mathcal X)
	:=
	\bm{\mathscr R}_2
	\left(
	\bm{\mathscr R}_1(\mathcal X)
	\times_1 A_1
	\times_2\cdots
	\times_{\widetilde d}A_{\widetilde d}
	\right)
	\times_1 B_1
	\times_2\cdots
	\times_{d'}B_{d'},
	\label{eq:intro_general_modewise_operator}
\end{equation}
where $\bm{\mathscr R}_1:\mathbb R^{n_1\times\cdots\times n_d}
\longrightarrow\mathbb R^{\widetilde n_1\times\cdots\times
	\widetilde n_{\widetilde d}}$ is a reshaping operator that reorganizes
the entries of the original $d$-mode tensor into a $\widetilde d$-mode
tensor. At the first stage, each matrix
$A_i\in\mathbb R^{\widetilde m_i\times\widetilde n_i}$,
$i\in[\widetilde d]$, is applied along the corresponding mode of the
reshaped tensor, yielding an intermediate tensor in
$\mathbb R^{\widetilde m_1\times\cdots\times
	\widetilde m_{\widetilde d}}$. The operator
$\bm{\mathscr R}_2:\mathbb R^{\widetilde m_1\times\cdots\times
	\widetilde m_{\widetilde d}}\longrightarrow
\mathbb R^{m_1'\times\cdots\times m_{d'}'}$ subsequently reorganizes
the first-stage output into a $d'$-mode tensor. At the second stage,
the matrices $B_j\in\mathbb R^{m_j\times m_j'}$, $j\in[d']$, are
applied along the corresponding modes to further compress the tensor,
producing the final measurement tensor in
$\mathbb R^{m_1\times\cdots\times m_{d'}}$. This framework can be extended to construct more general multi-stage operators by repeatedly alternating the reshaping and modewise-compression procedures.

Conventional vectorized measurements can be viewed as a particular
case of \eqref{eq:intro_general_modewise_operator}, in which
$\bm{\mathscr R}_1$ vectorizes the entire tensor and the first stage
compression is represented by a single dense matrix of size
$M\times\prod_{i=1}^{d}n_i$. In contrast, when $\bm{\mathscr R}_1$ performs
only a moderate reshaping, the modewise construction replaces this
massive matrix with a collection of substantially smaller component
matrices. It therefore requires fewer random parameters and reduces the
storage cost of the measurement operator. The individual mode products
are also naturally amenable to parallel implementation and preserve the
multimodal organization of the tensor throughout the first stage compression.
The storage savings and structural advantages noted above are also
characteristic of related approaches, including modewise oblivious
subspace embeddings, Kronecker-structured Johnson–Lindenstrauss
transforms, and sketching methods for large-scale Tucker
approximation~\cite{iwen2021lower,bamberger2022johnson,malik2018low}.

A key theoretical ingredient that connects such structured embeddings
with low-rank tensor recovery is the tensor restricted isometry
property (TRIP). Introduced in the analysis of tensor iterative hard
thresholding, TRIP extends the classical restricted isometry principle
to low-rank tensor models by requiring the measurement operator to
approximately preserve the Frobenius norm of tensors with prescribed
rank structure \cite{rauhut2017low}. It therefore provides a natural
condition under which the geometry of the low-rank tensor model is
retained after measurement. In the
modewise setting, when the component matrices satisfy suitable
restricted isometry conditions, the resulting one-stage and two-stage
operators likewise satisfy TRIP over tensors with bounded multilinear
rank \cite{haselby2023modewise}. In particular, these guarantees can be
established when the component matrices are chosen from either
subgaussian ensembles or Subsampled Orthogonal with Random Sign (SORS)
ensembles, the two measurement constructions considered later in this
work. Motivated by these results, we specialize the general framework
\eqref{eq:intro_general_modewise_operator} to the one-stage operator
$\bm{\mathscr L}_1=\operatorname{vec}\circ\bm{\mathscr A}\circ\bm{\mathscr R}$ and the
two-stage operator
$\bm{\mathscr L}_2=A_{\mathrm{2nd}}\circ\bm{\mathscr L}_1$, and incorporate them into the proposed Riemannian recovery method.
The corresponding TRIP-based convergence and sampling guarantees are
established in Section~\ref{sec:recovery_guarantee}.

\paragraph{Related work}
A variety of methods have been developed for low-rank tensor recovery.
Convex approaches commonly promote low multilinear rank by minimizing
the sum of the nuclear norms of tensor matricizations \cite{gandy2011tensor,huang2014provable,mu2014square},
but matricization may fail to fully exploit the intrinsic multilinear structure of the tensor and can lead to suboptimal sampling complexities. 
Tensor nuclear-norm formulations provide a more direct
alternative \cite{yuan2016tensor}, but evaluating or optimizing the corresponding
tensor norms is generally computationally demanding. More recently,
nonlocal tensor nuclear norms have been developed to exploit spatial
correlations in high-dimensional image recovery \cite{li2026importance}.
Nonconvex tensor low-rank surrogates have also been employed in imaging
inverse problems, such as tensor logarithmic Schatten-$p$ minimization
for 3D Poissonian image deblurring \cite{lu20243d}.
Nonconvex approaches have therefore received considerable attention due
to their lower computational cost and favorable recovery performance.
One class of methods parameterizes the unknown tensor through its core
tensor and factor matrices, and then applies alternating minimization or
gradient descent to the resulting factorized problem \cite{jain2014provable,xia2019polynomial,xia2021statistically}.
Another class of methods directly enforces the low-rank constraint
through iterative hard thresholding (IHT) \cite{blumensath2008iterative},
which has been extended to the tensor setting as tensor iterative hard
thresholding (TIHT) \cite{rauhut2017low}. However, applying the truncated
higher-order singular value decomposition (HOSVD) to a full-dimensional
tensor requires computing the singular value decomposition of each large
mode-$i$ matricization, whose column dimension grows rapidly with the
tensor order and can become prohibitively costly for high-dimensional
problems.
Riemannian gradient descent (RGD) overcomes this bottleneck by
projecting the Euclidean gradient onto the tangent space of the
fixed-multilinear-rank manifold at the current iterate, so that the
retraction step only needs to process a compact core tensor rather
than the full ambient tensor. Kressner et al.\ first developed this
Riemannian optimization framework for tensor completion, introducing
the tangent-space parametrization and HOSVD-based retraction on
which subsequent RGD methods rely \cite{kressner2014low}. Cai et
al.\ extended RGD to general linear measurement operators and
proved that, once initialized by one step of IHT, the resulting
iteration converges linearly to the underlying tensor under a
tensor restricted isometry condition, with a sampling complexity
optimal in the ambient dimension $n$ and every retraction confined
to core tensors of size at most $2r_1\times\cdots\times 2r_d$
\cite{cai2020provable}. More recent work has further extended RGD
along several directions, including its implicit regularization
behavior \cite{wang2023implicit} and variants based on preconditioned
or adapted Riemannian metrics \cite{zhang2025preconditioned,bian2024preconditioned,hamed2024riemannian}.

The theoretical recovery guarantees of the aforementioned methods rely on the TRIP. Existing TRIP results, however, are predominantly established for vectorized measurement operators. For a tensor in $\mathbb R^{n_1\times\cdots\times n_d}$, such an operator is represented by a matrix of size $M\times\prod_{i=1}^{d}n_i$. The storage of this matrix can exceed that
of the original tensor and therefore limits the practical use of
vectorization-based measurements in large scale settings.
To address this difficulty, Haselby et al.\ introduced
modewise measurement operators based on tensor reshaping and mode
products \cite{haselby2023modewise}. Compared with vectorized
measurements, modewise operators require fewer random parameters, reduce
the storage cost of the sensing map, and are compatible with structured
and parallel implementations. Building upon the modewise measurement framework in \cite{haselby2023modewise}, subsequent work extended this low-memory approach to Tucker approximation from one-pass streamed measurements \cite{haselby2026fast}.
The RGD and modewise measurement frameworks address two complementary
limitations of large-scale tensor recovery. RGD reduces the
computational cost of the optimization step by exploiting the geometry
of the fixed-rank tensor manifold, whereas modewise operators reduce
the storage cost of the measurement process. Despite its computational benefits, standard RGD suffers from a limitation, it applies a uniform global scaling to all orthogonal components of the Riemannian gradient. This approach neglects the heterogeneous scale disparities that often emerge between the core tensor and the factor matrices during optimization. Treating these distinct components with a single step size can lead to severe update imbalances, such as overshooting in the factors or stagnating in the core which degrades convergence speed. To mitigate such ill-conditioning, recent studies have explored Riemannian preconditioning and adapted metrics for low-rank matrix and tensor optimization \cite{hamed2024riemannian,mishra2016riemannian,zhang2025preconditioned,bian2024preconditioned}. Building upon the rationale of these geometric refinements, our motivation is to design a computationally lightweight yet effective scaling strategy that dynamically corrects these imbalances. To this end, we propose the adaptive block-weighted method, which adaptively rescales the core and factor gradient components of the Riemannian gradient according to their individual magnitudes, thereby harmonizing their updates and accelerating the overall recovery process.

\paragraph{Contributions}
In this study, we develop an adaptive block-weighted modewise
Riemannian gradient descent method for low-Tucker-rank tensor recovery.
The proposed approach combines the storage advantages of structured
modewise measurements with the geometric efficiency of Riemannian optimization on
the fixed-multilinear-rank tensor manifold. The main contributions are
summarized as follows.

\begin{itemize}
	\item
	We incorporate one-stage and two-stage modewise measurement
	operators into the RGD framework. The resulting recovery model
	avoids the storage of a single dense vectorized measurement matrix
	and remains compatible with modewise operators satisfying the
	tensor restricted isometry property.
	
	\item
	We propose a normalized adaptive block-weighted tangent operator
	that rescales the core and factor components of the
	Riemannian gradient according to their regularized Frobenius norms.
	The normalization maintains a unit average of the block weights and
	recovers the standard RGD direction when the Riemannian gradient
	components have equal Frobenius norms. Moreover, the weighting does not
	increase the multilinear-rank bound of the search direction or the size
	of the reduced core tensor used in the retraction.
	
	\item
	We establish a local linear convergence guarantee that quantifies the
	effect of the adaptive weighting. Under a
	$\operatorname{TRIP}(\delta_{2\boldsymbol r},2\boldsymbol r)$ condition and a
	sufficiently accurate initialization, the iterates remain in a
	prescribed neighborhood of the underlying tensor and converge to it at
	a linear rate for every admissible weight deviation.
\end{itemize}
\noindent\textbf{Paper Outline.}
The remainder of this paper is organized as follows.
In Section~\ref{sec:preliminaries}, we introduce the tensor notation, Tucker decomposition, fixed-rank tensor manifold, Riemannian gradient descent framework, and modewise measurement operators.
In Section~\ref{sec:adaptive_modewise_rgrad}, we present the proposed adaptive block-weighted modewise Riemannian gradient descent method together with the normalized weighting strategy and its efficient implementation.
In Section~\ref{sec:recovery_guarantee}, we establish the modewise TRIP conditions, local linear convergence guarantees, and sampling complexity of the proposed method.
Section~\ref{sec:experiments} presents numerical experiments under Gaussian and SORS measurements to demonstrate the recovery performance and computational efficiency of our method.
Finally, Section~\ref{sec:conclusion} concludes the paper and discusses directions for future work.

\section{Preliminaries and Problem Formulation}
\label{sec:preliminaries}

Throughout this paper, calligraphic letters, such as $\mathcal X$, are used to denote tensors, while capital Roman letters, such as $V$ and $A$, denote matrices. Bold lowercase letters, such as $\boldsymbol x$, denote vectors. Linear operators on tensors are denoted by bold script letters, such as $\bm{\mathscr{L}}$. Sets and spaces are denoted by calligraphic or blackboard bold letters, depending on the context. For a tensor $\mathcal X\in\mathbb R^{n_1\times\cdots\times n_d}$, its $(j_1,\ldots,j_d)$-th entry is denoted by $x_{j_1\cdots j_d}$ or $[\mathcal X]_{j_1,\ldots,j_d}$. For any positive integer $p$, we write $[p]:=\{1,\ldots,p\}$; in particular, $[d]$ denotes the set of tensor-mode indices. Integer multi-indices are denoted by bold lowercase letters, such as $\boldsymbol a=(a_1,\ldots,a_d)$. For two multi-indices $\boldsymbol a,\boldsymbol b\in\mathbb N^d$, the relation $\boldsymbol a\preceq\boldsymbol b$ means that $a_i\leq b_i$ for every $i\in[d]$. For a matrix $X$, its singular values are arranged in nonincreasing order, and $\sigma_j(X)$ denotes its $j$-th largest singular value. The tensor inner product and the associated Frobenius norm are denoted by $\langle\cdot,\cdot\rangle_F$ and $\|\cdot\|_F$, respectively. The identity operator is denoted by $\bm{\mathscr I}$, the adjoint of a linear operator $\bm{\mathscr L}$ is denoted by $\bm{\mathscr L}^{*}$, and its operator norm is denoted by $\|\bm{\mathscr L}\|$. In the remainder of this section, we introduce the basic tensor operations, the geometry of fixed-multilinear-rank tensors, and the measurement operators used throughout the paper.

\subsection{Tensor notation}
\label{subsec:tensor_notation_tucker}

For tensors $\mathcal X,\mathcal Y\in\mathbb R^{n_1\times\cdots\times n_d}$, the inner product is $\langle\mathcal X,\mathcal Y\rangle_F:=\sum_{i_1=1}^{n_1}\cdots\sum_{i_d=1}^{n_d}x_{i_1\cdots i_d}y_{i_1\cdots i_d}$
and Frobenius norm is $\|\mathcal X\|_F:=\sqrt{\langle\mathcal X,\mathcal X\rangle_F}$.

The mode-$i$ matricization of a tensor $\mathcal X\in\mathbb R^{n_1\times\cdots\times n_d}$ is defined as a linear mapping that unfolds $\mathcal{X}$ into a matrix $\mathcal X_{(i)}\in\mathbb R^{n_i\times N_i}$, where $N_i:=\prod_{j\ne i}n_j$. Specifically, this operation reshapes $\mathcal{X}$ such that each mode-$i$ fiber $\mathcal{X}_{j_1,\ldots,j_{i-1}, :, j_{i+1}, \ldots, j_d} \in \mathbb{R}^{n_i}$ constitutes a column of $\mathcal{X}_{(i)}$. Furthermore, we define the vectorization operator $\operatorname{vec}: \mathbb R^{n_1\times\cdots\times n_d} \to \mathbb R^{\prod_{j=1}^d n_j}$ as the canonical linear bijection that maps the tensor to a column vector, with its inverse denoted by $\operatorname{ten}_{p_1,\ldots,p_s} := \operatorname{vec}^{-1}$ for a target space $\mathbb{R}^{p_1 \times \cdots \times p_s}$.

For a tensor $\mathcal X \in \mathbb R^{n_1\times\cdots\times n_d}$ and a matrix $U\in\mathbb R^{p_i\times n_i}$, the mode-$i$ product, denoted by $\mathcal X\times_iU$, yields a tensor in $\mathbb R^{n_1\times\cdots\times n_{i-1} \times p_i \times n_{i+1}\times\cdots\times n_d}$. Its entries are explicitly defined by
\begin{equation*}
	\bigl(\mathcal X\times_iU\bigr)_{j_1,\ldots,j_{i-1},\ell,j_{i+1},\ldots,j_d}
	=
	\sum_{j_i=1}^{n_i}
	x_{j_1,\ldots,j_i,\ldots,j_d}
	U_{\ell,j_i}.
\end{equation*}
for all $(j_1,\ldots,j_{i-1},\ell,j_{i+1},\ldots,j_d)\in[n_1]\times\cdots\times[n_{i-1}]\times[p_i]\times[n_{i+1}]\times\cdots\times[n_d]$.
When applied to a rank-one tensor $\mathcal Y = \bigcirc_{j=1}^d\boldsymbol v^{(j)}$ composed of vectors $\boldsymbol v^{(j)}\in\mathbb R^{n_j}$, the mode-$i$ product acts exclusively on the $i$-th factor vector, yielding 
\[
\mathcal Y\times_iU = \bigl(\bigcirc_{j=1}^{i-1}\boldsymbol v^{(j)}\bigr) \circ \bigl(U\boldsymbol v^{(i)}\bigr) \circ \bigl(\bigcirc_{j=i+1}^{d}\boldsymbol v^{(j)}\bigr).
\]

\subsection{Tucker decomposition and fixed-rank tensor manifold}
\label{subsec:tucker_manifold}

Let
$\boldsymbol r=(r_1,\ldots,r_d)\in\mathbb N^{d}$
with $1\leq r_i\leq n_i$ for all $i\in[d]$.
The multilinear rank of a tensor
$\mathcal X\in\mathbb R^{n_1\times\cdots\times n_d}$
is defined by $\operatorname{mulrank}(\mathcal X):=\bigl(\operatorname{rank}(\mathcal X_{(1)}),\ldots,\operatorname{rank}(\mathcal X_{(d)})\bigr)$.
We say that $\mathcal X$ has multilinear rank at most
$\boldsymbol r$ if
$\operatorname{mulrank}(\mathcal X)\preceq\boldsymbol r$.
Equivalently, there exist subspaces $\mathcal U_i\subseteq\mathbb R^{n_i}$ and $\dim(\mathcal U_i)=r_i$
such that $\mathcal X\in\bigotimes_{i=1}^d\mathcal U_i$.

For each $i\in[d]$, let
$\{\boldsymbol v_1^{(i)},\ldots,\boldsymbol v_{r_i}^{(i)}\}$
be an orthonormal basis of $\mathcal U_i$, and define the corresponding factor matrix
$V^{(i)}:=\bigl[\boldsymbol v_1^{(i)},\ldots,\boldsymbol v_{r_i}^{(i)}\bigr]
\in\mathbb R^{n_i\times r_i}$.
Then there exists a core tensor
$\mathcal B\in\mathbb R^{r_1\times\cdots\times r_d}$
such that
\begin{align}
	\mathcal X
	=
	\mathcal B
	\times_1V^{(1)}
	\times_2\cdots
	\times_dV^{(d)}
	=
	\sum_{k_1=1}^{r_1}\cdots
	\sum_{k_d=1}^{r_d}
	\mathcal B(k_1,\ldots,k_d)
	\bigcirc_{i=1}^d
	\boldsymbol v_{k_i}^{(i)}.
	\label{eq:tucker_decomposition}
\end{align}
This representation is called an orthogonal Tucker decomposition of
$\mathcal X$. The factor matrices satisfy $\bigl(V^{(i)}\bigr)^TV^{(i)}=I_{r_i}$, for $i\in[d]$.
Moreover,
$\operatorname{mulrank}(\mathcal X)=\boldsymbol r$
if and only if the core tensor $\mathcal B$ has full multilinear rank
$\boldsymbol r$.

The collection of all such tensors forms a smooth embedded submanifold in the ambient Euclidean tensor space, denoted by 
\begin{equation*}
	\mathbb M_{\boldsymbol r} := \left\{ \mathcal X \in \mathbb R^{n_1\times\cdots\times n_d} : \operatorname{mulrank}(\mathcal X) = \boldsymbol r \right\},
\end{equation*}
whose dimension is given by $\dim(\mathbb M_{\boldsymbol r}) = \prod_{i=1}^d r_i + \sum_{i=1}^d (n_ir_i-r_i^2)$ \cite{kressner2014low,cai2020provable}.

By differentiating the factors of the orthogonal Tucker decomposition $\mathcal X = \mathcal B\times_{i\in[d]}V^{(i)} \in \mathbb M_{\boldsymbol r}$ and imposing the gauge conditions $(V^{(i)})^T\dot V^{(i)}=0$ to eliminate the infinitesimal nonuniqueness of the representation, the tangent space at $\mathcal X$ is characterized as
\begin{equation*}
	T_{\mathcal X}\mathbb M_{\boldsymbol r} =
	\Bigg\{ \dot{\mathcal B} \times_{i\in[d]}V^{(i)} + \sum_{k=1}^d \mathcal B \times_{j\in[d]\setminus\{k\}}V^{(j)} \times_k\dot V^{(k)} \;\Bigg|\; \dot{\mathcal B} \in \mathbb R^{r_1\times\cdots\times r_d},\, (V^{(k)})^T\dot V^{(k)}=0 \Bigg\}.
\end{equation*}
Under these gauge conditions, the tangent space naturally admits a mutually orthogonal direct sum decomposition:
\begin{equation}
	T_{\mathcal X}\mathbb M_{\boldsymbol r} = \mathcal S_{\mathcal X}^{(0)} \oplus \mathcal S_{\mathcal X}^{(1)} \oplus\cdots\oplus \mathcal S_{\mathcal X}^{(d)}
	\label{eq:tangent_direct_sum}
\end{equation}
with respect to the Frobenius inner product. Specifically, this consists of one core-variation block $\mathcal S_{\mathcal X}^{(0)} := \{ \dot{\mathcal B} \times_{i\in[d]}V^{(i)} \mid \dot{\mathcal B} \in \mathbb R^{r_1\times\cdots\times r_d} \}$ and $d$ factor-variation blocks $\mathcal S_{\mathcal X}^{(k)} := \{ \mathcal B \times_{j\in[d]\setminus\{k\}}V^{(j)} \times_k\dot V^{(k)} \mid \dot V^{(k)} \in \mathbb R^{n_k\times r_k},\, (V^{(k)})^T\dot V^{(k)}=0 \}$ for each $k\in[d]$.

For $k=0,\ldots,d$, let $\bm{\Pi}_{\mathcal X}^{(k)}:\mathbb R^{n_1\times\cdots\times n_d}\longrightarrow\mathcal S_{\mathcal X}^{(k)}$ denote the orthogonal projector onto the $k$-th tangent block $\mathcal S_{\mathcal X}^{(k)}$. Then the orthogonal projector from the ambient tensor space onto $T_{\mathcal X}\mathbb M_{\boldsymbol r}$ can be written as $\bm{\mathscr P}_{T_{\mathcal X}\mathbb M_{\boldsymbol r}}=\sum_{k=0}^d\bm{\Pi}_{\mathcal X}^{(k)}$.
Importantly, any tangent vector $\xi \in T_{\mathcal X}\mathbb M_{\boldsymbol r}$ inherits a strictly bounded structural complexity from this subspace configuration, yielding a multilinear rank bounded componentwise by $2\boldsymbol r$, i.e., $\operatorname{mulrank}(\xi) \preceq 2\boldsymbol r$ \cite{kressner2014low}.

\subsection{Riemannian gradient descent algorithm}
\label{subsec:standard_rgrad}

To minimize the least-squares objective $f(\mathcal X) = \frac{1}{2}\|\bm{\mathscr L}(\mathcal X)-\boldsymbol y\|_2^2$ over the fixed-rank manifold $\mathbb M_{\boldsymbol r}$, the standard RGD algorithm iteratively updates the tensor via gradient projection and manifold retraction \cite{cai2020provable}. At the $l$-th iterate $\mathcal X_l \in \mathbb M_{\boldsymbol r}$, the Euclidean gradient is given by $\mathcal G_l := \nabla f(\mathcal X_l) = \bm{\mathscr L}^*\bigl(\bm{\mathscr L}(\mathcal X_l)-\boldsymbol y\bigr)$.
Under the Frobenius metric, the Riemannian gradient is uniquely determined by projecting the Euclidean gradient onto the current tangent space $\mathcal S_l := T_{\mathcal X_l}\mathbb M_{\boldsymbol r}$, yielding $\bm{\mathscr{P}}_{\mathcal S_l}\mathcal G_l$. 

Moving along the negative Riemannian gradient direction, the exact line-search step size $\alpha_l$ is computed as
\begin{equation*}
	\alpha_l
	= \argmin_{\alpha\ge0} f\bigl(\mathcal X_l-\alpha \bm{\mathscr{P}}_{\mathcal S_l}\mathcal G_l\bigr)
	= \frac{\|\bm{\mathscr{P}}_{\mathcal S_l}\mathcal G_l\|_F^2}{\|\bm{\mathscr L}(\bm{\mathscr{P}}_{\mathcal S_l}\mathcal G_l)\|_2^2},
\end{equation*}
provided that $\bm{\mathscr{P}}_{\mathcal S_l}\mathcal G_l\ne 0$. The equality follows directly from the orthogonality of the projector $\bm{\mathscr{P}}_{\mathcal S_l}$, which yields $\langle\mathcal G_l,\bm{\mathscr{P}}_{\mathcal S_l}\mathcal G_l\rangle_F =\|\bm{\mathscr{P}}_{\mathcal S_l}\mathcal G_l\|_F^2$.

Because the intermediate tensor $\mathcal X_l - \alpha_l \bm{\mathscr{P}}_{\mathcal S_l}\mathcal G_l$ resides in the tangent space $\mathcal S_l$ and typically possesses a multilinear rank up to $2\boldsymbol r$, a retraction operator is required to pull it back onto the prescribed low-rank manifold $\mathbb M_{\boldsymbol r}$. A computationally efficient and widely adopted retraction is the truncated HOSVD operator $\bm{\mathscr H}_{\boldsymbol r}$, which retains the leading $r_i$ left singular vectors of each mode-$i$ matricization \cite{vannieuwenhoven2012new}. The resulting truncation procedure is summarized in Algorithm~\ref{alg:truncated_hosvd}.

\begin{algorithm}[htbp]
	\caption{Truncated HOSVD $\bm{\mathscr H}_{\boldsymbol r}$, \cite{de2000multilinear,tucker1966some}}
	\label{alg:truncated_hosvd}
	\begin{algorithmic}[1]
		\Require Tensor $\mathcal Y\in\mathbb R^{n_1\times\cdots\times n_d}$ and target rank $\boldsymbol r=(r_1,\ldots,r_d)$.
		\For{$i=1,\ldots,d$}
		\State Compute $V^{(i)}\in\mathbb R^{n_i\times r_i}$ as the $r_i$ dominant left singular vectors of $\mathcal Y_{(i)}$.
		\EndFor
		\State $\mathcal B\gets\mathcal Y\times_{i\in[d]}(V^{(i)})^T$.
		\State \Return $\bm{\mathscr H}_{\boldsymbol r}(\mathcal Y) :=\mathcal B\times_{i\in[d]}V^{(i)}$.
	\end{algorithmic}
\end{algorithm}

Unlike the matrix case, computing best low-rank tensor approximations is generally NP-hard \cite{hillar2013most}. A practical alternative is the truncated HOSVD, which serves as a quasi-projection onto the multilinear-rank-$\boldsymbol r$ manifold $\mathbb M_{\boldsymbol r}$ \cite{de2000multilinear,tucker1966some}. Specifically, letting $\bm{\mathscr P}_{\mathbb M_{\boldsymbol r}}$ denote the exact projection onto $\mathbb M_{\boldsymbol r}$, we have
\begin{equation}
	\|\mathcal Y-\bm{\mathscr H}_{\boldsymbol r}(\mathcal Y)\|_F
	\le
	\sqrt d\,
	\|\mathcal Y-\bm{\mathscr P}_{\mathbb M_{\boldsymbol r}}(\mathcal Y)\|_F,
	\qquad
	\mathcal Y\in\mathbb R^{n_1\times\cdots\times n_d}.
	\label{eq:hosvd_quasi_optimality}
\end{equation}
Consequently, initialized by $\mathcal X_0=\bm{\mathscr H}_{\boldsymbol r}(\bm{\mathscr L}^*\boldsymbol y)$, the standard RGD iteration is formulated as $\mathcal X_{l+1}=\bm{\mathscr H}_{\boldsymbol r}\left(\mathcal X_l-\alpha_l\bm{\mathscr{P}}_{\mathcal S_l}\mathcal G_l\right)$.
This formulation allows the retraction to be efficiently evaluated on a compact core tensor of size at most $2r_1\times\cdots\times 2r_d$, completely avoiding the HOSVD of the full-dimensional ambient tensor.

\subsection{Restricted isometry properties}
\label{subsec:trip_preliminaries}

We first recall the restricted isometry property (RIP) on a prescribed set, which will be used to characterize the component matrices of the modewise measurement operators introduced later. The definition applies to an arbitrary
subset of a normed vector space and is therefore not restricted to a
particular low-dimensional model.

\begin{definition}[RIP($\varepsilon,\mathcal S$) property]
	\label{def:set_rip}
	Let $\mathcal S$ be a subset of a normed vector space and let
	$\bm{\mathscr A}$ be a linear map. For
	$0<\varepsilon<1$, we say that $\bm{\mathscr A}$ satisfies the
	$\operatorname{RIP}(\varepsilon,\mathcal S)$ property if
	\begin{equation}
		(1-\varepsilon)\|\boldsymbol z\|^2
		\le
		\|\bm{\mathscr A}(\boldsymbol z)\|^2
		\le
		(1+\varepsilon)\|\boldsymbol z\|^2,
		\qquad
		\boldsymbol z\in\mathcal S.
		\label{eq:set_rip_definition}
	\end{equation}
\end{definition}

Thus, the RIP requires the linear map to approximately preserve the
norm of every element in the prescribed set $\mathcal S$. 
For tensor recovery, the corresponding norm-preservation property is
TRIP. In this case, the set consists of tensors whose multilinear rank is bounded by
a prescribed rank tuple.

\begin{definition}[TRIP($\delta,\boldsymbol r$) property, \cite{rauhut2017low}]
	\label{def:trip}
	Let $\boldsymbol r=(r_1,\ldots,r_d)\in\mathbb N^d$ and let
	$0<\delta<1$. A linear map $\bm{\mathscr A}$ is said to satisfy the
	$\operatorname{TRIP}(\delta,\boldsymbol r)$ property if
	\begin{equation*}
		(1-\delta)\|\mathcal X\|_F^2
		\le
		\|\bm{\mathscr A}(\mathcal X)\|_F^2
		\le
		(1+\delta)\|\mathcal X\|_F^2
	\end{equation*}
	for every tensor $\mathcal X\in
	\mathbb R^{n_1\times\cdots\times n_d}$ satisfying
	$\operatorname{mulrank}(\mathcal X)\preceq\boldsymbol r$.
\end{definition}

The TRIP therefore requires the measurement operator to act as an
approximate isometry on the set of tensors with multilinear rank
bounded by $\boldsymbol r$. In the following subsection, we introduce
the modewise measurement operators and specify restricted isometry
conditions on their component matrices that ensure the corresponding
TRIP.

\subsection{Modewise measurement}
\label{subsec:modewise_measurement_operators}

Following the multi-stage modewise measurement framework \eqref{eq:intro_general_modewise_operator} introduced in
\cite{haselby2023modewise}, a modewise measurement operator is a structured linear map constructed by alternating tensor reshaping and modewise matrix multiplications.
Before describing the corresponding measurement construction, we introduce two classes of structured sets. Suppose that $\mathcal X\in\mathbb R^{n\times\cdots\times n}$, let $q\geq2$ be an integer dividing $d$, and set $d':=d/q$.

\begin{definition}[The set $\mathcal S_{1,2}$, \cite{haselby2023modewise}]
	\label{def:S12}
	Consider a set of vectors in $\mathbb R^{n^q}$, let
	\[
	\mathcal S_1
	:=
	\left\{
	\bigotimes_{j=1}^{q}\boldsymbol v^{(j)}
	:
	\boldsymbol v^{(j)}\in\mathbb S^{n-1},
	\quad j\in[q]
	\right\}
	\subseteq
	\mathbb S^{n^q-1},
	\]
	and define $\mathcal{S}_2:=\biggl\{\frac{\boldsymbol{x} + \boldsymbol{y}}{\|\boldsymbol{x} + \boldsymbol{y}\|_2}\,\bigg|\,\boldsymbol{x},\boldsymbol{y} \in \mathcal{S}_1,\text{s.t. }\langle \boldsymbol{x},\boldsymbol{y} \rangle = 0\biggr\}$, then set $\mathcal S_{1,2}:=\mathcal S_1\cup\mathcal S_2\subseteq\mathbb S^{n^q-1}$.
\end{definition}

The HOSVD requires the factor vectors within each mode to be mutually
orthogonal. After modewise compression, however, the transformed factor
vectors are generally only approximately orthogonal. The following set
is therefore used to characterize the intermediate tensors produced by
the first-stage compression.

\begin{definition}[Nearly orthogonal tensors, \cite{haselby2023modewise}]
	\label{def:nearly_orthogonal_tensors}
	Let $R\geq 1$, $\mu\geq 0$, $\theta\geq 0$, and let $\widetilde{\boldsymbol r} = (\widetilde r_1,\ldots,\widetilde r_{d'}) \in \mathbb N^{d'}$. The set $\mathcal N_{R,\mu,\theta,\widetilde{\boldsymbol r}}$ consists of all tensors $\mathcal Y \in \mathbb R^{m_1\times\cdots\times m_{d'}}$ that admit the form
	\[
	\mathcal Y = \sum_{k_1=1}^{\widetilde r_1} \cdots \sum_{k_{d'}=1}^{\widetilde r_{d'}} \mathcal C(k_1,\ldots,k_{d'}) \bigcirc_{i=1}^{d'} \boldsymbol v_{k_i}^{(i)},
	\]
	where $\mathcal C \in \mathbb R^{\widetilde r_1 \times\cdots\times \widetilde r_{d'}}$
	is the core tensor and
	$\boldsymbol v_{k_i}^{(i)} \in \mathbb R^{m_i}$,
	$k_i\in[\widetilde r_i]$, $i\in[d']$, are the factor vectors.
	The factor vectors are not required to be mutually orthogonal or normalized.
	Instead, the core tensor and factor vectors satisfy the following conditions:
	\begin{enumerate}
		\item[(a)] $\bigl\| \boldsymbol v_{k_i}^{(i)} \bigr\|_2^2 \leq R$ for every $i\in[d']$ and $k_i\in[\widetilde r_i]$;
		
		\item[(b)] $\bigl| \langle \boldsymbol v_{k_i}^{(i)}, \boldsymbol v_{k_i'}^{(i)} \rangle \bigr| \leq \mu$ for every $i\in[d']$ and any $k_i\neq k_i'$;
		
		\item[(c)] $\|\mathcal C\|_F=1$;
		
		\item[(d)] The mode-$i$ subtensors of $\mathcal C$ are mutually orthogonal for every $i\in[d']$. More precisely, for any distinct indices $p\neq p'$ in $[\widetilde r_i]$, we have
		\[
		\sum_{k_1=1}^{\widetilde r_1} \cdots \sum_{k_{i-1}=1}^{\widetilde r_{i-1}} \sum_{k_{i+1}=1}^{\widetilde r_{i+1}} \cdots \sum_{k_{d'}=1}^{\widetilde r_{d'}} \mathcal C(k_1,\ldots,k_{i-1},p,k_{i+1},\ldots,k_{d'}) \mathcal C(k_1,\ldots,k_{i-1},p',k_{i+1},\ldots,k_{d'}) = 0.
		\]
		\item[(e)] $\|\mathcal Y\|_F\geq\theta$.
	\end{enumerate}
\end{definition}

The first stage groups
every $q$ consecutive modes of the original tensor and applies a
separate measurement matrix along each reshaped mode. The resulting
tensor is then vectorized and, when further dimension reduction is
required, compressed by a secondary measurement matrix. Specifically,
we define $\bm{\mathscr L}_1:=\operatorname{vec}\circ\bm{\mathscr A}\circ\bm{\mathscr R}$ and $\bm{\mathscr L}_2:=A_{\mathrm{2nd}}\circ\bm{\mathscr L}_1$,
where $\bm{\mathscr R}$ is the tensor reshaping operator,
$\bm{\mathscr A}$ is the first stage modewise compression map, and
$A_{\mathrm{2nd}}$ is the second stage measurement matrix.
We refer to $\bm{\mathscr L}_1$ and $\bm{\mathscr L}_2$ as the one-stage and
two-stage modewise measurement operators, respectively.

The complete construction can be summarized as
\[
\mathcal X
\xrightarrow{\ \bm{\mathscr R}\ }
\bm{\mathscr R}(\mathcal X)
\xrightarrow{\ \bm{\mathscr A}\ }
\bm{\mathscr A}(\bm{\mathscr R}(\mathcal X))
\xrightarrow{\ \operatorname{vec}\ }
\bm{\mathscr L}_1(\mathcal X)
\xrightarrow{\ A_{\mathrm{2nd}}\ }
\bm{\mathscr L}_2(\mathcal X).
\]
The first stage consists of tensor reshaping and modewise compression,
whereas the optional second stage applies an additional vectorized
compression to the first-stage output. The details of these operations
are given below.

\subsubsection{Step 1: tensor reshaping}

For simplicity, suppose that
$n_i=n$ and $r_i=r$ for all $i\in[d]$.
Consider the reshaping operator
\[
\bm{\mathscr R}:\bigotimes_{i=1}^{d}\mathbb R^n\longrightarrow\bigotimes_{s=1}^{d'}\mathbb R^{n^q}
\]
which groups every $q$ consecutive modes of a $d$-mode tensor into a
single mode. Thus, $\bm{\mathscr R}$ reduces the tensor order from $d$ to
$d'=d/q$.
$\bm{\mathscr R}$ is the unique linear operator whose action
on a rank-one tensor is given by
\begin{equation}
	\bm{\mathscr R}
	\left(
	\bigcirc_{i=1}^{d}\boldsymbol x^{(i)}
	\right)
	:=
	\bigcirc_{s=1}^{d'}
	\left(
	\bigotimes_{\ell=1+q(s-1)}^{qs}
	\boldsymbol x^{(\ell)}
	\right)
	=:
	\bigcirc_{s=1}^{d'}
	\widetilde{\boldsymbol x}^{(s)},
	\label{eq:reshaping_rank_one}
\end{equation}
where $\circ$ and $\otimes$ denote the outer product and the Kronecker
product, respectively, and $\widetilde{\boldsymbol v}^{(s)}\in\mathbb R^{n^q}$.
Suppose that $\mathcal X$ admits the Tucker decomposition
\eqref{eq:tucker_decomposition}. Then its reshaping
$\widetilde{\mathcal X}:=\bm{\mathscr R}(\mathcal X)
\in\bigotimes_{s=1}^{d'}\mathbb R^{n^q}$ can be written as
\begin{equation}
	\widetilde{\mathcal X}
	=
	\sum_{j_1=1}^{r^q}\cdots
	\sum_{j_{d'}=1}^{r^q}
	\widetilde{\mathcal B}(j_1,\ldots,j_{d'})
	\bigcirc_{s=1}^{d'}
	\widetilde{\boldsymbol v}_{j_s}^{(s)},
	\label{eq:reshaped_tensor_expansion}
\end{equation}
where $\operatorname{mulrank}\bigl(\widetilde{\mathcal X}\bigr)
\preceq\widetilde{\boldsymbol r}:=(r^q,\ldots,r^q)$,
$\widetilde{\mathcal B}$ is obtained by reshaping the original core
tensor $\mathcal B$, and each
$\widetilde{\boldsymbol v}_{j_s}^{(s)}$ is the Kronecker product of the
corresponding $q$ Tucker factor vectors of $\mathcal X$.
Since Kronecker products of orthonormal factor vectors remain
orthonormal, the vectors
$\{\widetilde{\boldsymbol v}_{j_s}^{(s)}\}_{j_s=1}^{r^q}$
form an orthonormal family for each $s\in[d']$. 

\subsubsection{Step 2: first-stage modewise compression}

The second step performs the first-stage compression by applying a
separate measurement matrix along each mode of the reshaped tensor.
For every $i\in[d']$, let $A_i\in\mathbb R^{m_i\times n^q}$, $m_i\ll n^q$. Define linear map $\bm{\mathscr A}:\mathbb R^{n^q\times\cdots\times n^q}\longrightarrow\mathbb R^{m_1\times\cdots\times m_{d'}}$
by
\begin{equation}
	\bm{\mathscr A}(\mathcal Y)
	:=
	\mathcal Y
	\times_1A_1
	\times_2\cdots
	\times_{d'}A_{d'}.
	\label{eq:first_stage_tensor_map}
\end{equation}
Thus, $\bm{\mathscr A}$ acts modewise on $d'$-mode tensors. The corresponding first-stage tensor-valued measurement map on the original tensor space is $\bm{\mathscr F}(\mathcal X):=\bm{\mathscr A}(\bm{\mathscr R}(\mathcal X))$. Let $\widetilde{\mathcal X}:=\bm{\mathscr R}(\mathcal X)$ be the reshaped tensor obtained in Step~1. By
\eqref{eq:reshaped_tensor_expansion}, 
using the action of a mode product on rank-one tensors, we obtain
\begin{align}
	\bm{\mathscr F}(\mathcal X)
	=
	\bm{\mathscr A}(\widetilde{\mathcal X})
	=
	\sum_{j_1=1}^{\widetilde r_1}
	\cdots
	\sum_{j_{d'}=1}^{\widetilde r_{d'}}
	\widetilde{\mathcal B}(j_1,\ldots,j_{d'})
	\bigcirc_{i=1}^{d'}
	\left(
	A_i\widetilde{\boldsymbol v}_{j_i}^{(i)}
	\right).
	\label{eq:first_stage_expansion}
\end{align}

For compatibility with the vector-valued measurement model used in the
recovery problem, we define the one-stage modewise measurement operator
by vectorizing the first-stage output:
\begin{equation*}
	\bm{\mathscr L}_1(\mathcal X)
	:=
	\operatorname{vec}\bigl(\bm{\mathscr F}(\mathcal X)\bigr)
	=
	\operatorname{vec}
	\bigl(
	\bm{\mathscr A}(\bm{\mathscr R}(\mathcal X))
	\bigr),
\end{equation*}
where $\bm{\mathscr L}_1(\mathcal X)\in\mathbb R^{M_{\mathrm{1st}}}$, $M_{\mathrm{1st}}:=\prod_{i=1}^{d'}m_i$. When $m_i=m$ for all $i\in[d']$, the intermediate measurement
dimension is $M_{\mathrm{1st}}=m^{d'}$.

Let $\operatorname{ten}_{m_1,\ldots,m_{d'}}:\mathbb R^{M_{\mathrm{1st}}}\longrightarrow\mathbb R^{m_1\times\cdots\times m_{d'}}$,
denote the inverse vectorization operator. Since the adjoint of a
mode-$i$ multiplication by $A_i$ is the mode-$i$ multiplication by
$A_i^T$, the adjoint of $\bm{\mathscr L}_1$ is
\begin{equation*}
	\bm{\mathscr L}_1^*(\boldsymbol z)
	=
	\bm{\mathscr R}^*
	\left(
	\operatorname{ten}_{m_1,\ldots,m_{d'}}(\boldsymbol z)
	\times_1A_1^T
	\times_2\cdots
	\times_{d'}A_{d'}^T
	\right).
\end{equation*}

The expansion \eqref{eq:first_stage_expansion} shows that the first-stage map acts on the reshaped factor vectors through the component matrices $A_i$. Accordingly, we impose $\operatorname{RIP}(\varepsilon,\mathcal S_{1,2})$ on these matrices, where $\mathcal S_{1,2}$ is defined in Definition~\ref{def:S12}. The following proposition shows that these componentwise RIP conditions imply the TRIP for the first-stage modewise measurement operator.

\begin{proposition}[Theorem 3.3, \cite{haselby2023modewise}]
	\label{prop:first_stage_modewise_trip}
	Let $r\geq2$ and $\boldsymbol r:=(r,\ldots,r)\in\mathbb N^{d}$.
	Suppose that the reshaping operator $\bm{\mathscr R}$ and the modewise
	map $\bm{\mathscr A}$ are defined by
	\eqref{eq:reshaping_rank_one} and
	\eqref{eq:first_stage_tensor_map}, respectively. Assume that, for
	every $i\in[d']$, the matrix $A_i\in\mathbb R^{m_i\times n^q}$
	satisfies the $\operatorname{RIP}(\varepsilon,\mathcal S_{1,2})$
	condition in \eqref{eq:set_rip_definition}. Define
	$\delta:=4d'r^d\varepsilon$ and suppose that $\delta<1$. Then the
	composite map $\bm{\mathscr A}\circ\bm{\mathscr R}$ satisfies the
	$\operatorname{TRIP}(\delta,\boldsymbol r)$ condition. More precisely,
	\begin{equation*}
		(1-\delta)\|\mathcal X\|_F^2
		\leq
		\left\|
		\bm{\mathscr A}(\bm{\mathscr R}(\mathcal X))
		\right\|_2^2
		\leq
		(1+\delta)\|\mathcal X\|_F^2
	\end{equation*}
	for every
	$\mathcal X\in\mathbb R^{n\times\cdots\times n}$
	satisfying $\operatorname{mulrank}(\mathcal X)\preceq\boldsymbol r$.
	
	Moreover, since vectorization preserves the Frobenius norm, the
	one-stage vector-valued operator $\bm{\mathscr L}_1=\operatorname{vec}\circ\bm{\mathscr A}\circ\bm{\mathscr R}$
	satisfies the equivalent bound
	\begin{equation*}
		(1-\delta)\|\mathcal X\|_F^2
		\leq
		\|\bm{\mathscr L}_1(\mathcal X)\|_2^2
		\leq
		(1+\delta)\|\mathcal X\|_F^2.
	\end{equation*}
\end{proposition}

\subsubsection{Step 3: secondary compression}

The first stage output $\bm{\mathscr F}(\mathcal X)=\bm{\mathscr A}(\bm{\mathscr R}(\mathcal X))\in\mathbb R^{m_1\times\cdots\times m_{d'}}$
retains a tensor structure. However, the transformed factor vectors $\left\{A_i\widetilde{\boldsymbol v}_{j}^{(i)}\right\}_{j=1}^{\widetilde r_i}$, $i\in[d']$
appearing in \eqref{eq:first_stage_expansion} are generally no longer
mutually orthogonal. Consequently,
\eqref{eq:first_stage_expansion} need not be an HOSVD of
$\bm{\mathscr F}(\mathcal X)$, and its multilinear rank cannot be inferred
directly from that representation.
This difficulty is addressed using the nearly orthogonal tensor class
introduced in Definition~\ref{def:nearly_orthogonal_tensors}. For the
theoretical statements below, we assume that
$m_i=m$ for all $i\in[d']$, so that $M_{\mathrm{1st}}=m^{d'}$.

\begin{proposition}[Lemma 5.11, \cite{haselby2023modewise}]
	\label{prop:first_stage_nearly_orthogonal}
	Let $\boldsymbol r=(r,\ldots,r)\in\mathbb N^d$ and $\widetilde{\boldsymbol r}=(r^q,\ldots,r^q)\in\mathbb N^{d'}$. Suppose that $\mathcal X\in\mathbb R^{n\times\cdots\times n}$ is a unit-norm tensor (i.e., $\|\mathcal X\|_F=1$) satisfying $\operatorname{mulrank}(\mathcal X) \preceq \boldsymbol r$. Assume that the first-stage modewise map $\bm{\mathscr A}$ is defined by \eqref{eq:first_stage_tensor_map}, and each matrix $A_i$ satisfies the $\operatorname{RIP}(\varepsilon,\mathcal S_{1,2})$ property for $i\in[d']$. If $\delta := 12d'r^d\varepsilon < 1$, then the first-stage output satisfies:
	\begin{equation*}
		\bm{\mathscr A}(\bm{\mathscr R}(\mathcal X))
		\in
		\mathcal N_{1+\varepsilon,\varepsilon,1-\delta/3,
		\widetilde{\boldsymbol r}}.
	\end{equation*}
\end{proposition}

Proposition~\ref{prop:first_stage_nearly_orthogonal} demonstrates that the first-stage measurements of all unit-norm tensors with multilinear rank bounded by $\boldsymbol r$ reside in the set of nearly orthogonal tensors $\mathcal N_{1+\varepsilon,\varepsilon,1-\delta/3,\widetilde{\boldsymbol r}}$. Consequently, this specific set defines the domain over which the secondary compression matrix must preserve Euclidean norms.

To achieve a more compact final measurement vector, we introduce a secondary measurement matrix $A_{\mathrm{2nd}} \in \mathbb R^{M_{\mathrm{2nd}}\times M_{\mathrm{1st}}}$. The cascaded two-stage modewise measurement operator is then defined as:
\begin{equation}
	\bm{\mathscr L}_2(\mathcal X) := A_{\mathrm{2nd}}\bm{\mathscr L}_1(\mathcal X) = A_{\mathrm{2nd}} \operatorname{vec}\bigl( \bm{\mathscr A}(\bm{\mathscr R}(\mathcal X)) \bigr) \in \mathbb R^{M_{\mathrm{2nd}}},
	\label{eq:two_stage_modewise_operator}
\end{equation}
with adjoint given by $\bm{\mathscr L}_2^*(\boldsymbol z) = \bm{\mathscr L}_1^*(A_{\mathrm{2nd}}^T\boldsymbol z)$ for any $\boldsymbol z\in\mathbb R^{M_{\mathrm{2nd}}}$.

The following proposition gives a sufficient condition for the
two-stage operator to satisfy the TRIP.

\begin{proposition}[Theorem 3.8, \cite{haselby2023modewise}]
	\label{prop:two_stage_modewise_trip}
	Let $r\geq2$, $\boldsymbol r = (r,\ldots,r)\in\mathbb N^d$, and $\widetilde{\boldsymbol r} = (r^q,\ldots,r^q)\in\mathbb N^{d'}$. Suppose $\bm{\mathscr R}$, $\bm{\mathscr A}$, and $\bm{\mathscr L}_2$ are defined by \eqref{eq:reshaping_rank_one}, \eqref{eq:first_stage_tensor_map}, and \eqref{eq:two_stage_modewise_operator}, respectively. Assume that each $A_i\in\mathbb R^{m_i\times n^q}$ satisfies the $\operatorname{RIP}(\varepsilon,\mathcal S_{1,2})$ property for $i\in[d']$, and let $\delta = 12d'r^d\varepsilon < 1$. If $A_{\mathrm{2nd}}$ satisfies the $\operatorname{RIP}\bigl(
	\delta/3,
	\operatorname{vec}(
	\mathcal N_{1+\varepsilon,\varepsilon,1-\delta/3, \widetilde{\boldsymbol r}}
	)\bigr)$ property, then $\bm{\mathscr L}_2$ satisfies the $\operatorname{TRIP}(\delta,\boldsymbol r)$ condition. Specifically,
	\begin{equation*}
		(1-\delta)\|\mathcal X\|_F^2 \leq \|\bm{\mathscr L}_2(\mathcal X)\|_2^2 \leq (1+\delta)\|\mathcal X\|_F^2
	\end{equation*}
	holds for every $\mathcal X\in\mathbb R^{n\times\cdots\times n}$ satisfying $\operatorname{mulrank}(\mathcal X) \preceq \boldsymbol r$.
\end{proposition}

Proposition~\ref{prop:two_stage_modewise_trip} shows that, under the prescribed RIP conditions, the secondary compression further reduces the measurement dimension while preserving the approximate isometry of the modewise measurement operator over tensors of bounded multilinear rank. In Section~\ref{subsec:modewise_trip_2r}, we specialize these guarantees to multilinear rank bounded by $2\boldsymbol r$ and present the corresponding measurement dimension bounds for sub-Gaussian and SORS matrices.

In the remainder of this paper, $\bm{\mathscr L}$ denotes either $\bm{\mathscr L}_1$ or $\bm{\mathscr L}_2$, depending on whether the one-stage or two-stage measurement model is utilized.

\section{Adaptive Block-Weighted Modewise RGD}
\label{sec:adaptive_modewise_rgrad}

To recover the underlying low-rank tensor from the compressed measurements, we formulate the task as a constrained least squares optimization problem over the fixed rank tensor manifold $\mathbb M_{\boldsymbol r}$:
\begin{equation}
	\min_{\mathcal X\in\mathbb M_{\boldsymbol r}} f(\mathcal X) := \frac{1}{2} \left\| \bm{\mathscr L}(\mathcal X)-\boldsymbol y \right\|_2^2.
	\label{eq:weighted_objective}
\end{equation}
In this section, we introduce the Adaptive Block-Weighted Modewise Riemannian Gradient Descent algorithm for solving the low-rank tensor recovery problem in \eqref{eq:weighted_objective}.

\subsection{Adaptive weighted tangent operator and algorithm}
\label{subsec:adaptive_tangent_operator}

We construct the adaptive weighted tangent operator from the gradient at each iteration. Let $\mathcal{G}_l := \bm{\mathscr L}^* \big(\bm{\mathscr L}(\mathcal{X}_l) - \boldsymbol{y}\big)$ be the Euclidean gradient of the objective function at the $l$-th iterate $\mathcal{X}_l$, and let $\mathcal{S}_l := T_{\mathcal X_l}\mathbb M_{\boldsymbol r}$ denote the tangent space of the manifold $\mathbb M_{\boldsymbol r}$ at $\mathcal{X}_l$. In our proposed algorithm, we retain the standard orthogonal tangent space decomposition and adaptively rescale its individual blocks.

As introduced in Section~\ref{subsec:tucker_manifold}, the tangent space $\mathcal S_l$ admits the mutually orthogonal direct-sum decomposition $\mathcal S_l=\mathcal S_l^{(0)}\oplus\mathcal S_l^{(1)}\oplus\cdots\oplus\mathcal S_l^{(d)}$, where $\bm{\Pi}_l^{(k)}$ denotes the orthogonal projector onto $\mathcal S_l^{(k)}$. Consequently, the orthogonal projector onto the full tangent space can be written as $\bm{\mathscr{P}}_{\mathcal S_l}=\sum_{k=0}^{d}\bm{\Pi}_l^{(k)}$.
By applying these block projectors to the Euclidean gradient $\mathcal{G}_l$, we define the gradient components as follows:
\begin{equation}
	\mathcal D_l=\bm{\Pi}_l^{(0)}\mathcal G_l,
	\qquad
	\mathcal W_l^{(k)}=\bm{\Pi}_l^{(k)}\mathcal G_l,
	\quad k=1,\ldots,d.
	\label{eq:block_gradient_components}
\end{equation}
This yields the standard decomposition of the Riemannian gradient:
\begin{equation}
	\bm{\mathscr{P}}_{\mathcal S_l}\mathcal G_l = \mathcal D_l + \sum_{k=1}^{d} \mathcal W_l^{(k)}. \label{eq:standard_gradient_decomposition}
\end{equation}

In standard RGD, all components in \eqref{eq:standard_gradient_decomposition} contribute to the Riemannian gradient with the same unit weight. However, the relative magnitudes of the core and factor gradient components, given by $\bm{\Pi}_l^{(k)}\mathcal G_l$ for $k=0,\ldots,d$, can vary considerably across iterations. Motivated by this observation, we introduce an adaptive weighting rule based on the Frobenius norms of these individual tangent-gradient components.
Specifically, we define the block magnitudes as
\begin{equation}
	s_{l,k} = \sqrt{ \left\| \bm{\Pi}_l^{(k)}\mathcal G_l \right\|_F^2 + \eta^2 }, \qquad k=0,\ldots,d, \label{eq:block_magnitude}
\end{equation}
where $\eta>0$ is a strictly positive regularization parameter. The regularization parameter $\eta>0$ keeps all regularized block magnitudes strictly positive and ensures that, as the tangent-gradient components vanish near convergence, the normalized weights approach one, thereby recovering the standard RGD scaling. To convert these regularized block magnitudes into comparable scaling factors for the tangent components, we normalize them so that the average weight remains equal to one. This leads to the following definition.

\begin{definition}[Normalized adaptive weights]
	\label{def:normalized_weights}
	At the $l$-th iteration, the normalized adaptive weight corresponding to the $k$-th tangent block is defined as
	\begin{equation}
		\omega_{l,k}
		=
		(d+1)
		\frac{s_{l,k}}{\sum_{j=0}^{d}s_{l,j}},
		\qquad k=0,\ldots,d.
		\label{eq:normalized_weights}
	\end{equation}
\end{definition}

The normalization in Definition~\ref{def:normalized_weights} preserves the
average scale of the tangent-gradient blocks. The weighting mechanism redistributes the relative
contributions of the tangent blocks without introducing an additional iteration-dependent global scaling.

The following elementary bound will be useful in Section~\ref{sec:recovery_guarantee}.

\begin{lemma}[Boundedness of the adaptive weights]
	\label{lem:weight_boundedness}
	Suppose that $\max_{0\le k\le d}\|\bm{\Pi}_l^{(k)}\mathcal G_l\|_F\le\Gamma$ on a set of iterates. Then, for all $k=0,\ldots,d$,
	\begin{equation*}
		0<
		\frac{\eta}{\sqrt{\Gamma^2+\eta^2}}
		=:\omega_{\min}
		\le
		\omega_{l,k}
		\le
		d+1
		=:\omega_{\max}.
	\end{equation*}
\end{lemma}

Based on the adaptive weights defined above, we obtain an adaptive block-weighted tangent operator as follows.

\begin{definition}[Adaptive block-weighted tangent operator]
	\label{def:adaptive_tangent_operator}
	Let $\{\omega_{l,k}\}_{k=0}^{d}$ be the normalized adaptive weights
	defined in Definition~\ref{def:normalized_weights}. The adaptive
	block-weighted tangent operator at $\mathcal X_l$ is defined as
	\begin{equation}
		\bm{\mathscr{P}}_l^{\boldsymbol\omega_l}
		=
		\sum_{k=0}^{d}
		\omega_{l,k}\bm{\Pi}_l^{(k)},
		\label{eq:adaptive_tangent_operator}
	\end{equation}
	The corresponding adaptive block-weighted tangent direction is defined by
	\begin{equation}
		\mathcal Z_l
		=
		\bm{\mathscr{P}}_l^{\boldsymbol\omega_l}\mathcal G_l.
		\label{eq:weighted_gradient}
	\end{equation}
\end{definition}

\begin{remark}
	If all tangent-gradient components have the same Frobenius norm, then
	Definition~\ref{def:normalized_weights} gives
	$\omega_{l,0}=\cdots=\omega_{l,d}=1$. Hence
	$\bm{\mathscr{P}}_l^{\boldsymbol\omega_l}=\bm{\mathscr{P}}_{\mathcal S_l}$, and
	$\mathcal Z_l=\bm{\mathscr{P}}_{\mathcal S_l}\mathcal G_l$.
	Thus, the proposed weighting scheme recovers the standard RGD search direction as a special case.
\end{remark}

Using \eqref{eq:block_gradient_components}, the weighted tangent
gradient admits the block representation
\begin{equation*}
	\mathcal Z_l
	=
	\omega_{l,0}\mathcal D_l
	+
	\sum_{k=1}^{d}
	\omega_{l,k}\mathcal W_l^{(k)}.
\end{equation*}

The adaptive operator \eqref{eq:adaptive_tangent_operator} also admits
a variable-metric interpretation \cite{mishra2016riemannian}. For $\xi,\nu\in \mathcal S_l$, define the
iteration-dependent inner product
\begin{equation}
	g_l^{\boldsymbol\omega_l}(\xi,\nu)
	=
	\sum_{k=0}^{d}
	\omega_{l,k}^{-1}
	\left\langle
	\bm{\Pi}_l^{(k)}\xi,
	\bm{\Pi}_l^{(k)}\nu
	\right\rangle_F.
	\label{eq:adaptive_metric}
\end{equation}
Since $\omega_{l,k}>0$ for every $k$, this defines a positive definite
inner product on $\mathcal S_l$.

For any $\xi\in \mathcal S_l$, using
\eqref{eq:weighted_gradient}--\eqref{eq:adaptive_metric}, we obtain
\begin{align*}
	g_l^{\boldsymbol\omega_l}(\mathcal Z_l,\xi)
	=
	\sum_{k=0}^{d}
	\omega_{l,k}^{-1}
	\left\langle
	\bm{\Pi}_l^{(k)}\mathcal Z_l,
	\bm{\Pi}_l^{(k)}\xi
	\right\rangle_F
	=
	\sum_{k=0}^{d}
	\left\langle
	\bm{\Pi}_l^{(k)}\mathcal G_l,
	\bm{\Pi}_l^{(k)}\xi
	\right\rangle_F
	=
	\left\langle
	\mathcal G_l,\xi
	\right\rangle_F
	=
	Df(\mathcal X_l)[\xi],
\end{align*}
where $Df(\mathcal X_l)[\xi]$ denotes the directional derivative of the objective function $f$ at $\mathcal X_l$ along the tangent direction $\xi$. Therefore, $\mathcal Z_l$ can be viewed as the gradient associated with
the iteration-dependent metric
\eqref{eq:adaptive_metric}. Since the weights depend on the current
gradient, this metric is understood as a local variable metric defined
at each iteration rather than a fixed metric prescribed globally on
$\mathbb M_{\boldsymbol r}$.

Having constructed the adaptive tangent direction $\mathcal Z_l$, we next determine the step size by minimizing the least-squares objective along the search direction $-\mathcal Z_l$. Since the objective is quadratic in $\alpha$, the exact line search admits a closed-form solution. Specifically,
\begin{equation}
	\alpha_l
	=
	\argmin_{\alpha\geq0}
	f(\mathcal X_l-\alpha\mathcal Z_l)
	=
	\frac{
		\left\langle
		\mathcal G_l,\mathcal Z_l
		\right\rangle_F
	}{
		\left\|
		\bm{\mathscr L}(\mathcal Z_l)
		\right\|_2^2
	}
	=
	\frac{
		\displaystyle
		\sum_{k=0}^{d}
		\omega_{l,k}
		\left\|
		\bm{\Pi}_l^{(k)}\mathcal G_l
		\right\|_F^2
	}{
		\left\|
		\bm{\mathscr L}(\mathcal Z_l)
		\right\|_2^2
	},
	\label{eq:exact_stepsize}
\end{equation}
provided that
$\bm{\mathscr L}(\mathcal Z_l)\neq0$.

In general, the numerator in \eqref{eq:exact_stepsize} is not equal to
$\|\mathcal Z_l\|_F^2$. By the orthogonality of the tangent blocks, $\left\|\mathcal Z_l\right\|_F^2=\sum_{k=0}^{d}\omega_{l,k}^2\left\|\bm{\Pi}_l^{(k)}\mathcal G_l\right\|_F^2$.
The equality $\left\langle\mathcal G_l,\mathcal Z_l\right\rangle_F=\|\mathcal Z_l\|_F^2$
holds only in the standard unweighted case. This distinction is
important because $\bm{\mathscr{P}}_l^{\boldsymbol\omega_l}$ is not an orthogonal projection
and will also play a role in the convergence analysis in
Section~\ref{sec:recovery_guarantee}.

Building on the adaptive block-weighted tangent operator and the associated variable metric introduced above, we formulate the block-weighted modewise RGD method for solving \eqref{eq:weighted_objective}. At each iteration, the residual is first mapped back to the ambient tensor space by $\bm{\mathscr L}^*$ to obtain the Euclidean gradient, which is then decomposed into its core and factor tangent components. The normalized adaptive weights are updated from the Frobenius norms of these components, and the corresponding block-weighted tangent direction is constructed according to Definition~\ref{def:adaptive_tangent_operator}. An exact line search is subsequently performed along this direction, followed by a truncated-HOSVD retraction onto $\mathbb M_{\boldsymbol r}$. The complete procedure is summarized in Algorithm~\ref{alg:abw_modewise_rgrad}.

\begin{algorithm}[t]
	\caption{Adaptive Block-Weighted Modewise RGD}
	\label{alg:abw_modewise_rgrad}
	\begin{algorithmic}[1]
		
		\State \textbf{Initialize:} 
		$\mathcal{X}_0 = \bm{\mathscr H}_{\boldsymbol r} \big( \bm{\mathscr L}^*(\boldsymbol{y}) \big)$
		
		\For{$l=0,1,\ldots$}
		
		\State \textbf{Compute gradient:} 
		$\mathcal{G}_l = \bm{\mathscr L}^* \big( \bm{\mathscr L}(\mathcal{X}_l) - \boldsymbol{y} \big)$
		
		\State \textbf{Update adaptive weights:} 
		Compute $\omega_{l,k}$ for $k=0,1,\ldots,d$ according to \eqref{eq:normalized_weights}
		
		\State \textbf{Compute search direction:} 
		Compute $\mathcal{Z}_l = \bm{\mathscr{P}}_l^{\boldsymbol\omega_l}\mathcal{G}_l$ according to \eqref{eq:adaptive_tangent_operator} and \eqref{eq:weighted_gradient}
		
		\State \textbf{Compute step size:} 
		Compute $\alpha_l$ according to \eqref{eq:exact_stepsize}
		
		\State \textbf{Update tensor:} 
		$\mathcal{X}_{l+1} = \bm{\mathscr H}_{\boldsymbol r} \big( \mathcal{X}_l - \alpha_l \mathcal{Z}_l \big)$
		
		\EndFor
		
		\State \textbf{Output:} $\mathcal{X}_{l+1}$ when the stopping criterion is met.
		
	\end{algorithmic}
\end{algorithm}

\subsection{Computation in Adaptive Block-Weighted RGD}

\subsubsection{Computation of the search direction}

We need to compute the $d+1$ tangent-gradient components introduced in \eqref{eq:block_gradient_components}. Because the direct sum decomposition of the tangent space $\mathcal{S}_l$ in \eqref{eq:tangent_direct_sum} is mutually orthogonal, the orthogonal projection onto $\mathcal{S}_l$ can be decomposed into independent least-squares subproblems for the core tensor component and each factor matrix component. 

Let $\mathcal{Y} \in \mathbb{R}^{n_1\times\cdots\times n_d}$ be an arbitrary tensor, its orthogonal projection onto $\mathcal{S}_l$ is given by the solution to the following optimization problem:
\begin{equation}
	\bm{\mathscr{P}}_{\mathcal S_l}\mathcal{Y}
	=
	\argmin_{\xi\in \mathcal{S}_l}
	\|\mathcal{Y}-\xi\|_F^2.
	\label{eq:global_tangent_ls}
\end{equation}
An arbitrary element $\xi \in \mathcal{S}_l$ admits the representation
\begin{equation}
	\xi
	=
	\dot{\mathcal B}
	\times_{i\in[d]}V_l^{(i)}
	+
	\sum_{k=1}^{d}
	\mathcal B_l
	\times_{j\in[d]\setminus\{k\}}V_l^{(j)}
	\times_k\dot V^{(k)},
	\label{eq:tangent_representation}
\end{equation}
where $V_l^{(k)},\dot V^{(k)}\in \mathbb{R}^{n_k \times r_k}$ satisfy $\left(V_l^{(k)}\right)^T\dot V^{(k)}=0$ for $k=1,\ldots,d$.

The $d+1$ terms in \eqref{eq:tangent_representation}, consisting of the core-variation term $\dot{\mathcal B}\times_{i\in[d]}V_l^{(i)}$ and the $d$ factor-variation terms $\mathcal B_l\times_{j\in[d]\setminus\{k\}}V_l^{(j)}\times_k\dot V^{(k)}$, $k=1,\ldots,d$, are mutually orthogonal. Consequently, the global least-squares problem \eqref{eq:global_tangent_ls} decouples into $d+1$ independent subproblems \cite{cai2020provable}.

\paragraph{Core Component}
The core variation $\dot{\mathcal{B}}_l$ is obtained by solving
\begin{equation*}
	\dot{\mathcal B}_l
	=
	\argmin_{\dot{\mathcal B}\in\mathbb R^{r_1\times\cdots\times r_d}}
	\left\|
	\mathcal Y
	-
	\dot{\mathcal B}
	\times_{i\in[d]}V_l^{(i)}
	\right\|_F^2.
\end{equation*}
The closed-form solution is $\dot{\mathcal B}_l=\mathcal Y\times_{j=1}^{d}\left(V_l^{(j)}\right)^T$.
Consequently, the projection onto the core tangent subspace is $\bm{\Pi}_l^{(0)}\mathcal Y=\dot{\mathcal B}_l\times_{j=1}^{d}V_l^{(j)}$.

\paragraph{Factor Components}
For $k\in[d]$, we define the Kronecker product of the factor matrices (excluding the $k$-th mode) as
\begin{equation*}
	\bar V_l^{(k)}
	=
	V_l^{(d)}
	\otimes
	\cdots
	\otimes
	V_l^{(k+1)}
	\otimes
	V_l^{(k-1)}
	\otimes
	\cdots
	\otimes
	V_l^{(1)},
\end{equation*}
The $k$-th factor variation is obtained by solving the constrained matrix least-squares problem
\begin{equation*}
	\dot V_l^{(k)}
	=
	\argmin_{\substack{
			\dot V^{(k)}\in\mathbb R^{n_k\times r_k}\\
			\left(V_l^{(k)}\right)^T\dot V^{(k)}=0}}
	\left\|
	\mathcal Y_{(k)}
	-
	\dot V^{(k)}
	(\mathcal B_l)_{(k)}
	\left(\bar V_l^{(k)}\right)^T
	\right\|_F^2.
\end{equation*}
Let $\bm{\mathscr P}_{V_l^{(k)\perp}}=I_{n_k}-V_l^{(k)}\left(V_l^{(k)}\right)^T$
denote the orthogonal projector onto the orthogonal complement of the column space of $V_l^{(k)}$. The closed-form solution is
\begin{equation}
	\dot V_l^{(k)}
	=
	\bm{\mathscr P}_{V_l^{(k)\perp}}
	\mathcal Y_{(k)}
	\bar V_l^{(k)}
	\left(\mathcal B_l\right)_{(k)}^\dagger,
	\qquad
	k=1,\ldots,d,
	\label{eq:factor_closed_form}
\end{equation}
where $\mathcal{Y}_{(k)}$ and $(\mathcal{B}_l)_{(k)}$ denote the mode-$k$ matricizations of $\mathcal{Y}$ and $\mathcal{B}_l$, respectively, and $(\mathcal{B}_l)_{(k)}^\dagger$ is the Moore--Penrose pseudoinverse of $(\mathcal{B}_l)_{(k)}$.
The corresponding projection onto the $k$-th factor tangent block is
$\bm{\Pi}_l^{(k)}\mathcal Y=\mathcal B_l\times_{j\in[d]\setminus\{k\}}V_l^{(j)}\times_k \dot V_l^{(k)}$.

By setting $\mathcal Y=\mathcal G_l$, we obtain the core and factor tangent-gradient components
$\mathcal D_l:=\bm{\Pi}_l^{(0)}\mathcal G_l
=\dot{\mathcal B}_l\times_{i\in[d]}V_l^{(i)}$
and
$\mathcal W_l^{(k)}:=\bm{\Pi}_l^{(k)}\mathcal G_l
=\mathcal B_l\times_{j\in[d]\setminus\{k\}}V_l^{(j)}
\times_k\dot V_l^{(k)}$
for $k=1,\ldots,d$. Consequently, the adaptive block-weighted tangent direction admits the representation
\begin{equation*}
	\mathcal Z_l
	=
	\omega_{l,0}
	\dot{\mathcal B}_l
	\times_{i\in[d]}V_l^{(i)}
	+
	\sum_{k=1}^{d}
	\omega_{l,k}
	\mathcal B_l
	\times_{j\in[d]\setminus\{k\}}V_l^{(j)}
	\times_k\dot V_l^{(k)}.
\end{equation*}

\begin{remark}[Computation and scaling of the block magnitudes]
	\label{rem:block_magnitude}
	
	The adaptive weights introduced in Definition~\ref{def:normalized_weights} are determined by the Frobenius norms of the tangent-gradient blocks. These block norms can be efficiently evaluated without forming the corresponding full tangent tensors in the ambient space. For the core component $\mathcal D_l = \dot{\mathcal B}_l \times_{i\in[d]}V_l^{(i)}$, exploiting the orthonormality of the factor matrices yields $\|\mathcal D_l\|_F = \|\dot{\mathcal B}_l\|_F$. For the factor components, the $k$-th component $\mathcal W_l^{(k)} = \mathcal B_l \times_{j\in[d]\setminus\{k\}}V_l^{(j)} \times_k\dot V_l^{(k)}$ has the mode-$k$ matricization $(\mathcal W_l^{(k)})_{(k)}=\dot V_l^{(k)}(\mathcal B_l)_{(k)}(\bar V_l^{(k)})^T$.
	Because $\bar V_l^{(k)}$ has orthonormal columns, it follows that $\|\mathcal W_l^{(k)}\|_F = \| \dot V_l^{(k)}(\mathcal B_l)_{(k)} \|_F$.
	Thus, the block magnitudes in \eqref{eq:block_magnitude} can be computed as
	$s_{l,0}=\sqrt{\|\dot{\mathcal B}_l\|_F^2+\eta^2}$ and
	$s_{l,k}=\sqrt{\|\dot V_l^{(k)}(\mathcal B_l)_{(k)}\|_F^2+\eta^2}$,
	$k=1,\ldots,d$.
	This formulation avoids constructing the full ambient tangent tensors
	when computing the adaptive weights. Moreover,
	$\dot{\mathcal B}_l$ and $\dot V_l^{(k)}$ belong to different blocks
	of the Tucker parametrization, so their parameter-space norms are not
	directly comparable. The quantities
	$\|\mathcal D_l\|_F$ and $\|\mathcal W_l^{(k)}\|_F$ instead measure
	the corresponding perturbations in the common ambient Frobenius norm.
	
	Since the adaptive weighting is applied after the orthogonal
	tangent-space projection, the least-squares subproblems for computing
	$\dot{\mathcal B}_l$ and $\dot V_l^{(k)}$ are unchanged. The additional
	work consists only of evaluating the $d+1$ block magnitudes and
	normalizing the corresponding scalar weights.
\end{remark}

\subsubsection{Computation of the Retraction}
\label{subsubsec:efficient_retraction}

We now detail the computation of the retraction step in Algorithm~\ref{alg:abw_modewise_rgrad}. Let $\mathcal Y_l = \mathcal X_l - \alpha_l\mathcal Z_l$ denote the intermediate tensor in the tangent space $\mathcal{S}_l$. To obtain the next iterate, we retract $\mathcal Y_l$ from the tangent space $\mathcal{S}_l$ back to the smooth multilinear-rank manifold $\mathbb M_{\boldsymbol r}$ via the truncated HOSVD operator, i.e., $\mathcal X_{l+1} = \bm{\mathscr H}_{\boldsymbol r}(\mathcal Y_l)$.
A direct implementation would require applying the truncated HOSVD
to the full ambient tensor $\mathcal Y_l$. Using the tangent-space
representation, the same retraction can instead be computed from a
reduced core tensor \cite{cai2020provable}. Since both $\mathcal X_l$ and the search direction $\mathcal Z_l$ reside in $\mathcal{S}_l$, the updated tensor $\mathcal Y_l$ belongs to $\mathcal{S}_l$. Consequently, its multilinear rank is bounded by $2\boldsymbol{r}$ (i.e., $\operatorname{mulrank}(\mathcal Y_l) \preceq 2\boldsymbol r$). Specifically, expanding the terms yields
\begin{align*}
	\mathcal Y_l
	&=
	\mathcal X_l-\alpha_l\mathcal Z_l
	=
	\mathcal X_l-\alpha_l\Biggl(
	\omega_{l,0}\dot{\mathcal B}_l\times_{i\in[d]}V_l^{(i)}
	+
	\sum_{k=1}^d
	\omega_{l,k}\,
	\mathcal B_l
	\times_{j\in[d]\setminus\{k\}}V_l^{(j)}
	\times_k\dot V_l^{(k)}
	\Biggr)
	\nonumber\\
	&=
	\left(
	\mathcal B_l-\alpha_l\omega_{l,0}\dot{\mathcal B}_l
	\right)
	\times_{i\in[d]}V_l^{(i)}
	-
	\alpha_l\sum_{k=1}^d
	\omega_{l,k}\,
	\mathcal B_l
	\times_{j\in[d]\setminus\{k\}}V_l^{(j)}
	\times_k\dot V_l^{(k)}
	\coloneqq
	\mathcal C_l
	\times_{i\in[d]}
	\begin{bmatrix}
		V_l^{(i)} & \dot V_l^{(i)}
	\end{bmatrix}.
\end{align*}
Thus, $\mathcal Y_l$ admits an augmented Tucker representation parameterized by a sparse block core tensor $\mathcal C_l \in \mathbb R^{2r_1\times\cdots\times2r_d}$. The principal block of this core tensor is given by $[\mathcal C_l]_{1:r_1,\ldots,1:r_d} = \mathcal B_l - \alpha_l\omega_{l,0}\dot{\mathcal B}_l$. Furthermore, for each $k=1,\ldots,d$, the block corresponding to the variation of the $k$-th factor matrix is $[\mathcal C_l]_{1:r_1,\ldots,r_k+1:2r_k,\ldots,1:r_d} = -\alpha_l\omega_{l,k}\mathcal B_l$.
The weights $\omega_{l,k}$ scale only the nonzero blocks of
$\mathcal C_l$ and do not introduce additional factor directions.
Hence the multilinear-rank bound is unchanged.

To compute $\bm{\mathscr H}_{\boldsymbol r}(\mathcal{Y}_l)$ efficiently, we first compute the thin QR factorizations of the augmented factor matrices:
\begin{equation*}
	\begin{bmatrix}
		V_l^{(i)} & \dot V_l^{(i)}
	\end{bmatrix}
	=
	Q_l^{(i)}R_l^{(i)},
	\quad
	i=1,\ldots,d.
\end{equation*}
Using these orthogonal factors, $\mathcal Y_l$ can be equivalently expressed as
\begin{equation*}
	\mathcal Y_l
	=
	\mathcal C_l
	\times_{i\in[d]}
	\left(
	Q_l^{(i)}R_l^{(i)}
	\right)
	=
	\left(
	\mathcal C_l
	\times_{i\in[d]}R_l^{(i)}
	\right)
	\times_{i\in[d]}Q_l^{(i)}
	\coloneqq
	\widetilde{\mathcal C}_l
	\times_{i\in[d]}Q_l^{(i)}.
\end{equation*}
where $\widetilde{\mathcal C}_l = \mathcal C_l \times_{i\in[d]}R_l^{(i)} \in \mathbb R^{2r_1\times\cdots\times2r_d}$. Since the matrices $Q_l^{(i)}$ possess orthonormal columns for all $i=1,\ldots,d$, the dominant mode subspaces of $\mathcal Y_l$ exactly correspond to the leading singular vectors of the respective matricizations of the reduced core tensor $\widetilde{\mathcal C}_l$. Specifically, for each mode $i=1,\ldots,d$, let $\widetilde U_l^{(i)}$ denote the matrix of left singular vectors of
$(\widetilde{\mathcal C}_l)_{(i)}$.
The updated factor matrices are obtained via
\begin{equation*}
	V_{l+1}^{(i)}
	=
	Q_l^{(i)}
	\left[
	\widetilde U_l^{(i)}
	\right]_{:,1:r_i},
	\qquad
	i=1,\ldots,d.
\end{equation*}
The updated core tensor can be computed entirely in the reduced coordinates as
\begin{equation*}
	\mathcal B_{l+1}
	=
	\widetilde{\mathcal C}_l
	\times_{i\in[d]}
	\left(
	\left[
	\widetilde U_l^{(i)}
	\right]_{:,1:r_i}
	\right)^T.
\end{equation*}
Consequently, the final retracted tensor is formulated as 
\begin{equation*}
	\mathcal X_{l+1}=\mathcal B_{l+1}\times_{i\in[d]}V_{l+1}^{(i)}=\bm{\mathscr H}_{\boldsymbol r}(\mathcal Y_l).
\end{equation*}
Thus, the truncated HOSVD need only be applied to the reduced tensor
$\widetilde{\mathcal C}_l\in
\mathbb R^{2r_1\times\cdots\times2r_d}$, together with $d$ thin QR
factorizations of the augmented factor matrices.

The adaptive weighting leaves the reduced-core retraction unchanged
and adds only the evaluation of $d+1$ block magnitudes and their scalar
normalization. The measurement operator and its adjoint are applied
through the modewise implementations of $\bm{\mathscr L}$ and $\bm{\mathscr L}^*$.
Hence the search direction retains multilinear rank at most
$2\boldsymbol r$, and the truncated HOSVD is performed on a tensor of
size $2r_1\times\cdots\times2r_d$.

\section{Convergence Analysis}
\label{sec:recovery_guarantee}

In this section, we establish the local linear convergence guarantee of Algorithm~\ref{alg:abw_modewise_rgrad}. In particular, we prove that Algorithm~\ref{alg:abw_modewise_rgrad} converges linearly to the underlying unknown tensor $\mathcal{T}$, provided that the modewise measurement operator satisfies the TRIP. Furthermore, we establish the sampling complexity of the proposed approach.

\subsection{Modewise Tensor Restricted Isometry Property}
\label{subsec:modewise_trip_2r}
For the explicit modewise TRIP and sampling-complexity bounds, we restrict to the balanced setting $n_i=n$, $r_i=r$ for $i\in[d]$.
We first specify the order of TRIP required for the subsequent theoretical analysis. In the adaptive block-weighted tangent operator $\bm{\mathscr{P}}_l^{\boldsymbol\omega_l} = \sum_{k=0}^{d}\omega_{l,k}\bm{\Pi}_l^{(k)}$ introduced in Section~\ref{sec:adaptive_modewise_rgrad}, $\bm{\Pi}_l^{(k)}$ is the orthogonal projector onto the $k$-th component of $\mathcal S_l$. Since scalar multiplication by $\omega_{l,k}$ does not alter the underlying subspace, then $\omega_{l,k}\bm{\Pi}_l^{(k)}\mathcal{G}_l\in \mathcal S_l^{(k)}$ for $k=0,\ldots,d$. Consequently, the weighted projected gradient satisfies $\mathcal{Z}_l = \bm{\mathscr{P}}_l^{\boldsymbol\omega_l}\mathcal{G}_l \in \mathcal S_l$ with $\operatorname{mulrank}(\mathcal{Z}_l) \preceq 2\boldsymbol{r}$, which implies that the intermediate tensor $\mathcal Y_l = \mathcal{X}_l-\alpha_l\mathcal{Z}_l$ has a multilinear rank of at most $2\boldsymbol{r}$. The proposed adaptive weighting scheme does not increase the multilinear rank order required by the TRIP. 

Next, we describe the modewise measurement operators that satisfy this TRIP assumption. We first consider the one-stage modewise measurement operator $\bm{\mathscr L}_1$ introduced in Section~\ref{subsec:modewise_measurement_operators}, with modewise measurement matrices $A_i\in\mathbb R^{m\times n^q}$, $i=1,\ldots,d'$. The following proposition gives the corresponding TRIP guarantee for tensors of multilinear rank at most $2\boldsymbol r$.

\begin{proposition}[One-stage modewise TRIP at rank $2\boldsymbol r$, \cite{haselby2023modewise}]
	\label{prop:one_stage_trip_2r}
	Suppose $r\ge 1$, $q\ge 2$, and $0 < \zeta < 1$, and define $d':=d/q$. If each measurement matrix $A_i$ satisfies the $\operatorname{RIP}(\varepsilon,\mathcal{S}_{1,2})$ property, and $\delta_{2\boldsymbol r}=4d'(2r)^d\varepsilon<1$, then the first-stage measurement map $\bm{\mathscr L}_1$ respects the $\operatorname{TRIP}(\delta_{2\boldsymbol r},2\boldsymbol r)$ property, i.e.,
	\begin{equation*}
		(1-\delta_{2\boldsymbol r})\|\mathcal{X}\|_F^2 \le \|\bm{\mathscr L}_1(\mathcal{X})\|_2^2 \le (1+\delta_{2\boldsymbol r})\|\mathcal{X}\|_F^2
	\end{equation*}
	holds for all tensors $\mathcal{X}$ with multilinear rank at most $2\boldsymbol{r}$.
	
	Furthermore, this conclusion holds with probability at least $1-\zeta$ provided that either of the following measurement ensembles is utilized:
	\begin{enumerate}[label=(\roman*)]
		\item If the entries of each $A_i$ are properly normalized i.i.d.sub-Gaussian random variables, it is sufficient that
		\begin{equation*}
			m\ge
			C\delta_{2\boldsymbol r}^{-2}(2r)^{2d}
			\max\left\{
			\frac{nd^2\log q}{q},
			\frac{d^2}{q^2}\log\left(\frac{d}{q\zeta}\right)
			\right\}.
		\end{equation*}
		
		\item If each $A_i$ is a SORS matrix, it is sufficient that
		\begin{equation}
			m\ge
			C_1\delta_{2\boldsymbol r}^{-2}(2r)^{2d}
			\frac{nd^2\log q}{q}\,\Psi_1,
			\label{eq:one_stage_sors_2r}
		\end{equation}
		where the logarithmic factor $\Psi_1$ is defined as
		\begin{equation*}
			\Psi_1
			:=
			\log\left(\frac{2d}{q\zeta}\right)
			\log\left(\frac{2en^qd}{q\zeta}\right)
			\log^2\left[
			C_2\delta_{2\boldsymbol r}^{-2}(2r)^{2d}
			\frac{nd^2\log q}{q}
			\log\left(\frac{2d}{q\zeta}\right)
			\right].
		\end{equation*}
	\end{enumerate}
\end{proposition}

We next consider the two-stage modewise measurement operator $\bm{\mathscr L}_2$ introduced in Section~\ref{subsec:modewise_measurement_operators}, with second-stage matrix $A_{\mathrm{2nd}}\in\mathbb R^{M_{\mathrm{2nd}}\times m^{d'}}$. The following proposition gives the corresponding TRIP guarantee for tensors of multilinear rank at most $2\boldsymbol r$.

\begin{proposition}[Two-stage modewise TRIP at rank $2\boldsymbol r$, \cite{haselby2023modewise}]
	\label{prop:two_stage_trip_2r}
	Suppose $r\ge 1$, $q\ge 2$, and $0 < \zeta < 1$, and define $d':=d/q$. Assume that each first-stage measurement matrix $A_i$ satisfies the $\operatorname{RIP}(\varepsilon,\mathcal{S}_{1,2})$ property, and let $\delta_{2\boldsymbol r}=12d'(2r)^d\varepsilon<1$. If the second-stage measurement matrix $A_{\mathrm{2nd}}$ satisfies the associated RIP condition with reshaped rank $((2r)^q,\ldots,(2r)^q)$, then the two-stage measurement map $\bm{\mathscr L}_2$ respects the $\operatorname{TRIP}(\delta_{2\boldsymbol r},2\boldsymbol r)$ property, i.e.,
	\begin{equation*}
		(1-\delta_{2\boldsymbol r})\|\mathcal{X}\|_F^2 \le \|\bm{\mathscr L}_2(\mathcal{X})\|_2^2 \le (1+\delta_{2\boldsymbol r})\|\mathcal{X}\|_F^2
	\end{equation*}
	holds for all tensors $\mathcal{X}$ with multilinear rank at most $2\boldsymbol{r}$.
	
	Furthermore, this conclusion holds with probability at least $1-\zeta$ provided that either of the following measurement ensembles is utilized:
	\begin{enumerate}[label=(\roman*)]
		\item If the measurement matrices consist of properly normalized i.i.d.\ sub-Gaussian random variables, it is sufficient that
		\begin{equation*}
			m\ge
			C\delta_{2\boldsymbol r}^{-2}(2r)^{2d}
			\max\left\{
			\frac{nd^2\log q}{q},
			\frac{d^2}{q^2}\log\left(\frac{2d}{q\zeta}\right)
			\right\},
		\end{equation*}
		and
		\begin{align*}
			M_{\mathrm{2nd}}
			\ge C\delta_{2\boldsymbol r}^{-2}
			\max\Bigg\{&
			\left(
			\frac{(2r)^d q+d m(2r)^q}{q}
			\right)
			\log\left(\frac dq+1\right)
			\nonumber\\
			&+
			\frac{d m(2r)^q}{q}
			\log\left(1+\delta_{2\boldsymbol r}(2r)^d\right)
			+
			\frac{d^2m(2r)^q\delta_{2\boldsymbol r}}{q^2},
			\log\left(\frac{2}{\zeta}\right)
			\Bigg\}.
		\end{align*}
		
		\item If the measurement matrices are SORS matrices, the first-stage requirement on $m$ is identical to \eqref{eq:one_stage_sors_2r}. For the second stage, it is sufficient that
		\begin{align}
			M_{\mathrm{2nd}}
			\ge C\delta_{2\boldsymbol r}^{-2}
			\Bigg[&
			\left(
			\frac{(2r)^d q+d m(2r)^q}{q}
			\right)
			\log\left(\frac dq+1\right)
			\nonumber\\
			&+
			\frac{d m(2r)^q}{q}
			\log\left(1+\delta_{2\boldsymbol r}(2r)^d\right)
			+
			\frac{d^2m(2r)^q\delta_{2\boldsymbol r}}{q^2}
			\Bigg]\Psi_2,
			\label{eq:two_stage_second_sors_2r}
		\end{align}
		where the logarithmic factor $\Psi_2$ is defined as
		\begin{align*}
			\Psi_2
			:={}&
			\log^2\Bigg(
			\frac{c_1}{\delta_{2\boldsymbol r}^2}
			\log\left(\frac4\zeta\right)
			\Bigg[
			\left(
			\frac{(2r)^d q+d m(2r)^q}{q}
			\right)
			\log\left(\frac dq+1\right)
			\nonumber\\
			&\hspace{32mm}
			+\frac{d m(2r)^q}{q}
			\log\left(1+\delta_{2\boldsymbol r}(2r)^d\right)
			+\frac{d^2m(2r)^q\delta_{2\boldsymbol r}}{q^2}
			\Bigg]
			\Bigg)\nonumber\\
			&\times
			\log\left(\frac4\zeta\right)
			\log\left(\frac{4em}{\zeta}\right).
		\end{align*}
	\end{enumerate}
\end{proposition}

\subsection{Main Theorems}
\label{subsec:main_local_convergence}

We now establish the local linear convergence of the proposed adaptive block-weighted RGD iteration. To quantify the effect of the adaptive weighting, we first measure the deviation of the adaptive weights from the unit weights used in the standard RGD direction. Specifically, we define the maximal weight deviation at the $l$-th iteration as $\rho_l:=\max_{0\le k\le d}|\omega_{l,k}-1|$.
To bound this deviation, let $\Gamma\ge0$ be an upper bound on the Frobenius norms of the tangent-gradient components, and define $S_{\Gamma}:=\sqrt{\Gamma^2+\eta^2}$. Then $S_{\Gamma}$ provides an upper bound on the corresponding regularized block magnitudes, and $\rho_\eta(\Gamma):=\frac{d(S_{\Gamma}-\eta)}{S_{\Gamma}+d\eta}$ gives a uniform upper bound on the deviation of the normalized adaptive weights from one.
The weight deviation can further be controlled in terms of the current reconstruction error. As shown in Appendix~\ref{app:complete_convergence_proofs}, the $\operatorname{TRIP}(\delta_{2\boldsymbol r},2\boldsymbol r)$ condition yields an upper bound on the tangent-gradient components. Hence, whenever $\|\mathcal X_l-\mathcal T\|_F\le R$, we have $\rho_l\le\rho_\eta\!\left((1+\delta_{2\boldsymbol r})R\right)$. The auxiliary geometric estimates, the TRIP-based bounds, and the complete convergence proof are provided in Appendix~\ref{app:complete_convergence_proofs}.

\begin{theorem}[Local linear convergence of adaptive block-weighted modewise RGD]
	\label{thm:weighted_local_linear_convergence}
	Let $\mathcal T\in\mathbb M_{\boldsymbol r}$ and
	$\boldsymbol y=\bm{\mathscr L}(\mathcal T)$. Assume that the measurement
	operator $\bm{\mathscr L}$ satisfies
	$\operatorname{TRIP}(\delta_{2\boldsymbol r},2\boldsymbol r)$ with
	$\delta_{2\boldsymbol r}\in(0,1)$.
	Fix a local radius $R>0$ and an admissible weight-deviation level
	$\rho_\star\in(0,1)$. Suppose that the initial iterate
	$\mathcal X_0\in\mathbb M_{\boldsymbol r}$ satisfies
	$\|\mathcal X_0-\mathcal T\|_F<R$, and choose the regularization
	parameter $\eta>0$ such that
	$\rho_\eta\!\left((1+\delta_{2\boldsymbol r})R\right)\le\rho_\star$.
	Define
	\begin{equation*}
		\gamma_R
		:=
		\frac{
			2(\sqrt d+1)
		}{
			(1-\delta_{2\boldsymbol r})(1-\rho_\star)
		}
		\left[
		\delta_{2\boldsymbol r}+\rho_\star
		+
		\frac{
			(2^d-1)(1+\delta_{2\boldsymbol r}\rho_\star)
		}{
			\min_{1\le i\le d}\sigma_{r_i}(\mathcal T_{(i)})
		}R
		\right].
	\end{equation*}
	If $\gamma_R<1$, then the iterates generated by
	Algorithm~\ref{alg:abw_modewise_rgrad} satisfy
	\begin{equation*}
		\|\mathcal X_l-\mathcal T\|_F
		\le
		\gamma_R^{\,l}
		\|\mathcal X_0-\mathcal T\|_F
		<
		R,
		\qquad l\ge0.
	\end{equation*}
	Moreover, the adaptive weights satisfy
	$\rho_l\le\rho_\star$ for all $l\ge0$.
\end{theorem}

\begin{corollary}[Contraction under the standard RGD initialization]
	\label{cor:weighted_standard_initialization}
	Suppose that the initialization satisfies
	\begin{equation}
		\|\mathcal{X}_{0}-\mathcal{T}\|_F
		\le
		(\sqrt d+1)
		\delta_{2\boldsymbol r}\|\mathcal T\|_F.
		\label{eq:weighted_initial_error}
	\end{equation}
	Fix $\rho_\star\in(0,1)$ and choose $\eta$ such that $\rho_\eta\!\left((1+\delta_{2\boldsymbol r})(\sqrt d+1)\delta_{2\boldsymbol r}\|\mathcal T\|_F\right)\le\rho_\star$.
	Then Theorem~\ref{thm:weighted_local_linear_convergence} applies with
	$R=(\sqrt d+1)\delta_{2\boldsymbol r}\|\mathcal T\|_F$. In particular, the contraction factor becomes
	\begin{equation}
		\begin{split}
			\gamma_{\mathrm w}
			:=
			\frac{
				2(\sqrt d+1)
			}{
				(1-\delta_{2\boldsymbol r})(1-\rho_\star)
			}
			\Bigg[
			&\delta_{2\boldsymbol r}+\rho_\star\\
			&+
			(2^d-1)(1+\delta_{2\boldsymbol r}\rho_\star)
			(\sqrt d+1)\delta_{2\boldsymbol r}
			\frac{\|\mathcal T\|_F}{\min_{1\le i\le d}\sigma_{r_i}(\mathcal{T}_{(i)})}
			\Bigg].
		\end{split}
		\label{eq:weighted_gamma}
	\end{equation}
	If $\gamma_{\mathrm w}<1$, then $\|\mathcal X_l-\mathcal T\|_F\le\gamma_{\mathrm w}^{\,l}\|\mathcal X_0-\mathcal T\|_F$.
\end{corollary}

The following results further characterize the dependence of the local
convergence bound on the weight deviation and the regularization parameter,
as well as the asymptotic behavior of the adaptive weights.

\begin{remark}[Relation to the unweighted contraction factor]
	\label{rem:weighting_no_gain}
	Setting $\rho_\star=0$ in \eqref{eq:weighted_gamma} yields
	\begin{equation*}
		\gamma_{\mathrm u}
		:=
		\gamma_{\mathrm w}\big|_{\rho_\star=0}
		=
		\frac{2(\sqrt d+1)\delta_{2\boldsymbol r}}{1-\delta_{2\boldsymbol r}}
		\left[
		1+(2^d-1)(\sqrt d+1)
		\frac{\|\mathcal T\|_F}{\min_{1\le i\le d}\sigma_{r_i}(\mathcal T_{(i)})}
		\right],
	\end{equation*}
	which agrees with the corresponding unweighted bound. Moreover,
	$\gamma_{\mathrm w}$ is increasing with respect to $\rho_\star$ on
	$[0,1)$. Indeed, writing
	$\gamma_{\mathrm w}
	=\frac{2(\sqrt d+1)}{1-\delta_{2\boldsymbol r}}
	\frac{N(\rho_\star)}{1-\rho_\star}$, where
	$N(\rho_\star)
	=\delta_{2\boldsymbol r}+\rho_\star
	+(2^d-1)(\sqrt d+1)\delta_{2\boldsymbol r}
	\frac{\|\mathcal T\|_F}
	{\min_{1\le i\le d}\sigma_{r_i}(\mathcal T_{(i)})}
	(1+\delta_{2\boldsymbol r}\rho_\star)$, we have
	\[
	\frac{\partial}{\partial\rho_\star}
	\left(
	\frac{N(\rho_\star)}{1-\rho_\star}
	\right)
	=
	\frac{
		(1+\delta_{2\boldsymbol r})
		\left[
		1+
		(2^d-1)(\sqrt d+1)\delta_{2\boldsymbol r}
		\frac{\|\mathcal T\|_F}
		{\min_{1\le i\le d}\sigma_{r_i}(\mathcal T_{(i)})}
		\right]
	}{
		(1-\rho_\star)^2
	}
	>0.
	\]
	Hence $\gamma_{\mathrm w}\ge\gamma_{\mathrm u}$, with equality at
	$\rho_\star=0$.
\end{remark}

\begin{remark}[Explicit choice of the regularization parameter]
	\label{rem:eta_choice}
	The condition
	$\rho_\eta((1+\delta_{2\boldsymbol r})R)\le\rho_\star$ in
	Theorem~\ref{thm:weighted_local_linear_convergence}
	can be ensured by an explicit choice of $\eta$.
	Since $\sqrt{\Gamma^2+\eta^2}+\eta\ge2\eta$ and
	$\sqrt{\Gamma^2+\eta^2}+d\eta\ge(1+d)\eta$,
	\eqref{eq:rho_eta_bound} gives
	\begin{equation}
		\rho_\eta(\Gamma)
		=
		\frac{d\,\Gamma^2}
		{\bigl(\sqrt{\Gamma^2+\eta^2}+\eta\bigr)
			\bigl(\sqrt{\Gamma^2+\eta^2}+d\eta\bigr)}
		\le
		\frac{d\,\Gamma^2}{2(1+d)\eta^2}.
		\label{eq:rho_eta_explicit}
	\end{equation}
	Therefore, with $\Gamma=(1+\delta_{2\boldsymbol r})R$, the sufficient condition $\eta\ge(1+\delta_{2\boldsymbol r})R\sqrt{\frac{d}{2(1+d)\rho_\star}}$
	guarantees $\rho_l\le\rho_\star$ whenever
	$\|\mathcal X_l-\mathcal T\|_F\le R$.
\end{remark}

\begin{proposition}[Asymptotic weight deviation and contraction rate]
	\label{prop:asymptotic_degeneration}
	Let the assumptions of
	Theorem~\ref{thm:weighted_local_linear_convergence} hold with
	$\gamma_R<1$, so that $\mathcal X_l\to\mathcal T$. Then the adaptive
	weights satisfy
	\begin{equation}
		\rho_l
		\le
		\frac{d(1+\delta_{2\boldsymbol r})^2}{2(1+d)\eta^2}
		\|\mathcal X_l-\mathcal T\|_F^2,
		\label{eq:rho_second_order}
	\end{equation}
	that is, the weight deviation is second order in the reconstruction
	error, and the asymptotic contraction rate obeys
	\begin{equation}
		\limsup_{l\to\infty}
		\frac{\|\mathcal X_{l+1}-\mathcal T\|_F}{\|\mathcal X_l-\mathcal T\|_F}
		\le
		\frac{2(\sqrt d+1)\delta_{2\boldsymbol r}}{1-\delta_{2\boldsymbol r}},
		\label{eq:asymptotic_rate}
	\end{equation}
	which agrees with the corresponding asymptotic bound for the unweighted modewise RGD iteration.
\end{proposition}
\begin{proof}
	The block-gradient estimate \eqref{eq:trip_block_gradient_bound} gives
	$\max_{0\le k\le d}\|\bm{\Pi}_l^{(k)}\mathcal G_l\|_F
	\le(1+\delta_{2\boldsymbol r})\|\mathcal X_l-\mathcal T\|_F=:\Gamma_l$.
	Applying \eqref{eq:rho_eta_explicit} with
	$\Gamma=\Gamma_l$ yields \eqref{eq:rho_second_order}.
	Since $\gamma_R<1$ forces $\|\mathcal X_l-\mathcal T\|_F\to0$, we also
	have $\rho_l\to0$. Substituting $\rho_l\to0$ and
	$\|\mathcal X_l-\mathcal T\|_F\to0$ into the one-step recurrence
	\eqref{eq:weighted_one_step_recurrence_rhol} leaves the limiting factor
	$2(\sqrt d+1)\delta_{2\boldsymbol r}/(1-\delta_{2\boldsymbol r}
	)$, proving
	\eqref{eq:asymptotic_rate}.
\end{proof}

\begin{corollary}[Sampling complexity for linear convergence]
	\label{cor:sampling_complexity_convergence}
	Assume the balanced setting $n_i=n$ and $r_i=r$ for $i\in[d]$. Let $q\ge2$ divide $d$, $d'=d/q$, $\zeta\in(0,1)$ be the failure probability, and $\rho_\star\in(0,1)$ be the target deviation level. Suppose $\delta_\star\in(0,1)$ satisfies
	\begin{equation}
		2(\sqrt d+1)
		\left[
		\delta_\star+\rho_\star
		+
		(2^d-1)(1+\delta_\star\rho_\star)
		(\sqrt d+1)\delta_\star
		\frac{\|\mathcal T\|_F}
		{\min_{1\le i\le d}\sigma_{r_i}(\mathcal T_{(i)})}
		\right]
		<
		(1-\delta_\star)(1-\rho_\star),
		\label{eq:admissible_delta_star}
	\end{equation}
	and choose $\eta>0$ such that $\rho_\eta \!\left( (1+\delta_\star)(\sqrt d+1)\delta_\star\|\mathcal T\|_F \right) \le \rho_\star$. Let $\boldsymbol y=\bm{\mathscr L}(\mathcal T)$ and $\mathcal X_0 = \bm{\mathscr H}_{\boldsymbol r}(\bm{\mathscr L}^*\boldsymbol y)$ be the initialization of Algorithm~\ref{alg:abw_modewise_rgrad}. Then the following hold:
	
	\begin{enumerate}[label=(\roman*)]
		\item For the one-stage sub-Gaussian modewise measurement map, if
		\begin{equation*}
			m \ge C\delta_\star^{-2}(2r)^{2d} \max\left\{ \frac{nd^2\log q}{q}, \frac{d^2}{q^2} \log\left(\frac{d}{q\zeta}\right) \right\},
		\end{equation*}
		then with probability at least $1-\zeta$, Algorithm~\ref{alg:abw_modewise_rgrad} converges linearly to $\mathcal T$:
		\[
		\|\mathcal X_l-\mathcal T\|_F \le \gamma_{\mathrm w}^{\,l}\|\mathcal X_0-\mathcal T\|_F,
		\]
		where $\gamma_{\mathrm w}<1$ is defined in \eqref{eq:weighted_gamma}. The total measurement dimension is $M_{\mathrm{1st}}=m^{d'}$.
		
		\item For the two-stage sub-Gaussian modewise measurement map, if
		\begin{equation*}
			m \ge C\delta_\star^{-2}(2r)^{2d} \max\left\{ \frac{nd^2\log q}{q}, \frac{d^2}{q^2} \log\left(\frac{2d}{q\zeta}\right) \right\},
		\end{equation*}
		and
		\begin{align*}
			M_{\mathrm{2nd}} \ge C\delta_\star^{-2} \max\Bigg\{& \left( \frac{(2r)^d q+d m(2r)^q}{q} \right) \log\left(\frac dq+1\right) \nonumber\\
			&+ \frac{d m(2r)^q}{q} \log\left( 1+\delta_\star(2r)^d \right) + \frac{ d^2m(2r)^q\delta_\star }{ q^2 }, \log\left(\frac2\zeta\right) \Bigg\},
		\end{align*}
		then with probability at least $1-\zeta$, Algorithm~\ref{alg:abw_modewise_rgrad} converges linearly to $\mathcal T$ with contraction factor $\gamma_{\mathrm w}<1$:
		\[
		\|\mathcal X_l-\mathcal T\|_F \le \gamma_{\mathrm w}^{\,l}\|\mathcal X_0-\mathcal T\|_F.
		\]
	\end{enumerate}
	Analogous convergence guarantees hold for SORS measurement ensembles by substituting the respective sampling conditions with \eqref{eq:one_stage_sors_2r} and \eqref{eq:two_stage_second_sors_2r}.
\end{corollary}

\begin{proof}
	By Propositions~\ref{prop:one_stage_trip_2r} and
	\ref{prop:two_stage_trip_2r}, the stated sampling conditions ensure
	that $\bm{\mathscr L}$ satisfies
	$\operatorname{TRIP}(\delta_\star,2\boldsymbol r)$ with probability
	at least $1-\zeta$. Under this TRIP condition, the truncated-HOSVD
	initialization satisfies
	$\|\mathcal X_0-\mathcal T\|_F
	\le(\sqrt d+1)\delta_\star\|\mathcal T\|_F$.
	The prescribed choice of $\eta$ controls the adaptive weight
	deviation, while \eqref{eq:admissible_delta_star} ensures
	$\gamma_{\mathrm w}<1$. The conclusion therefore follows from
	Corollary~\ref{cor:weighted_standard_initialization}.
\end{proof}

\begin{remark}[Measurement dimension and rank dependence]
	For the one-stage construction, $m$ is linear in $n$ when
	the remaining parameters are fixed, but the total number of
	scalar outputs is $M_{\mathrm{1st}}=m^{d/q}$ and is therefore
	superlinear in $n$ whenever $d/q>1$. In the two-stage
	construction, the final measurement dimension
	$M_{\mathrm{2nd}}$ is linear in $n$ up to the
	displayed rank and logarithmic factors.
\end{remark}

\section{Numerical Experiments}
\label{sec:experiments}

In this section, we evaluate the proposed normalized block-weighted
modewise RGD method on synthetic low-multilinear-rank tensor recovery
problems. The unknown fourth-order tensors are generated randomly in
Tucker form. We consider both balanced and nonuniform tensor dimensions
and multilinear ranks. The two-stage measurement operator is constructed
from either Gaussian or SORS matrices. The intermediate modewise dimension
is denoted by $m$, while $p$ denotes the final target dimension. For
comparison, we also test the unweighted modewise RGD method and the
corresponding weighted and unweighted RGD methods with dense vectorized
measurements. In the figures, ``w'' and ``u'' denote the weighted and
unweighted variants, respectively, and ``vec'' denotes vectorized
measurements. For each parameter setting, 20 independent trials are
performed. A trial is regarded as successful if
\[
\frac{\|\mathcal X_l-\mathcal T\|_F}{\|\mathcal T\|_F}<10^{-2}
\]
within 1000 iterations. The reported number of iterations is averaged
over the 20 trials, with 1000 used as the iteration cap.

Figures~\ref{fig:gaussian_success} and \ref{fig:gaussian_iters} report
the results for Gaussian measurements. Four combinations of tensor
dimensions and multilinear ranks are considered. In all cases, the
successful recovery rate exhibits a clear transition as $p$ increases.
A larger intermediate dimension $m$ generally shifts the transition to
a smaller $p$ and reduces the number of iterations. The weighted modewise
method is comparable to, and often better than, its unweighted counterpart,
especially near the recovery threshold and for smaller values of $m$. 

\begin{figure}[!tbp]
	\centering
	\includegraphics[width=0.98\linewidth]{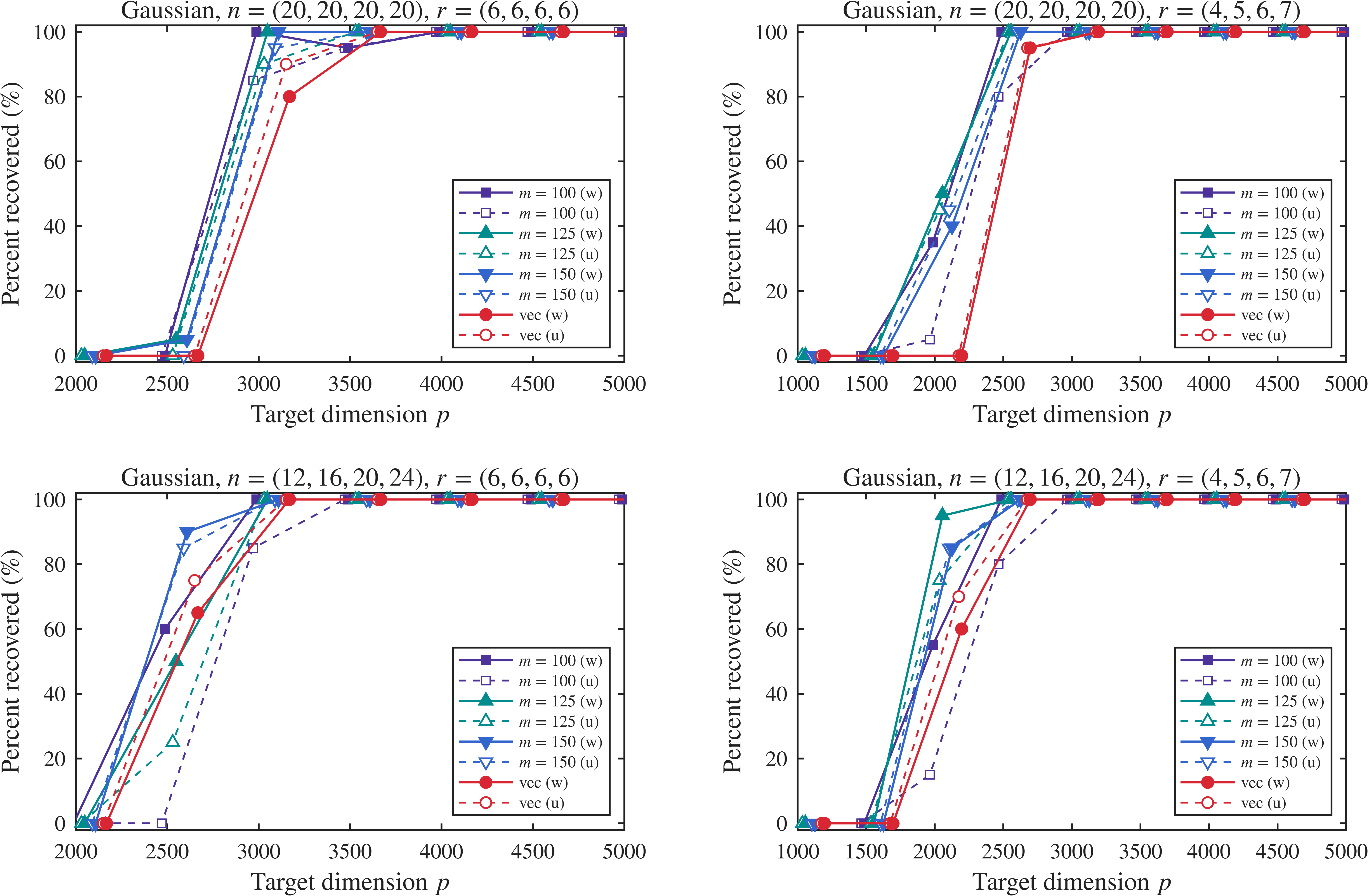}
	\caption{Successful recovery rates for Gaussian measurements. Each point is computed from 20 independent trials, and recovery is declared successful when the relative error is below $10^{-2}$ within 1000 iterations.}
	\label{fig:gaussian_success}
\end{figure}

\begin{figure}[!tbp]
	\centering
	\includegraphics[width=0.98\linewidth]{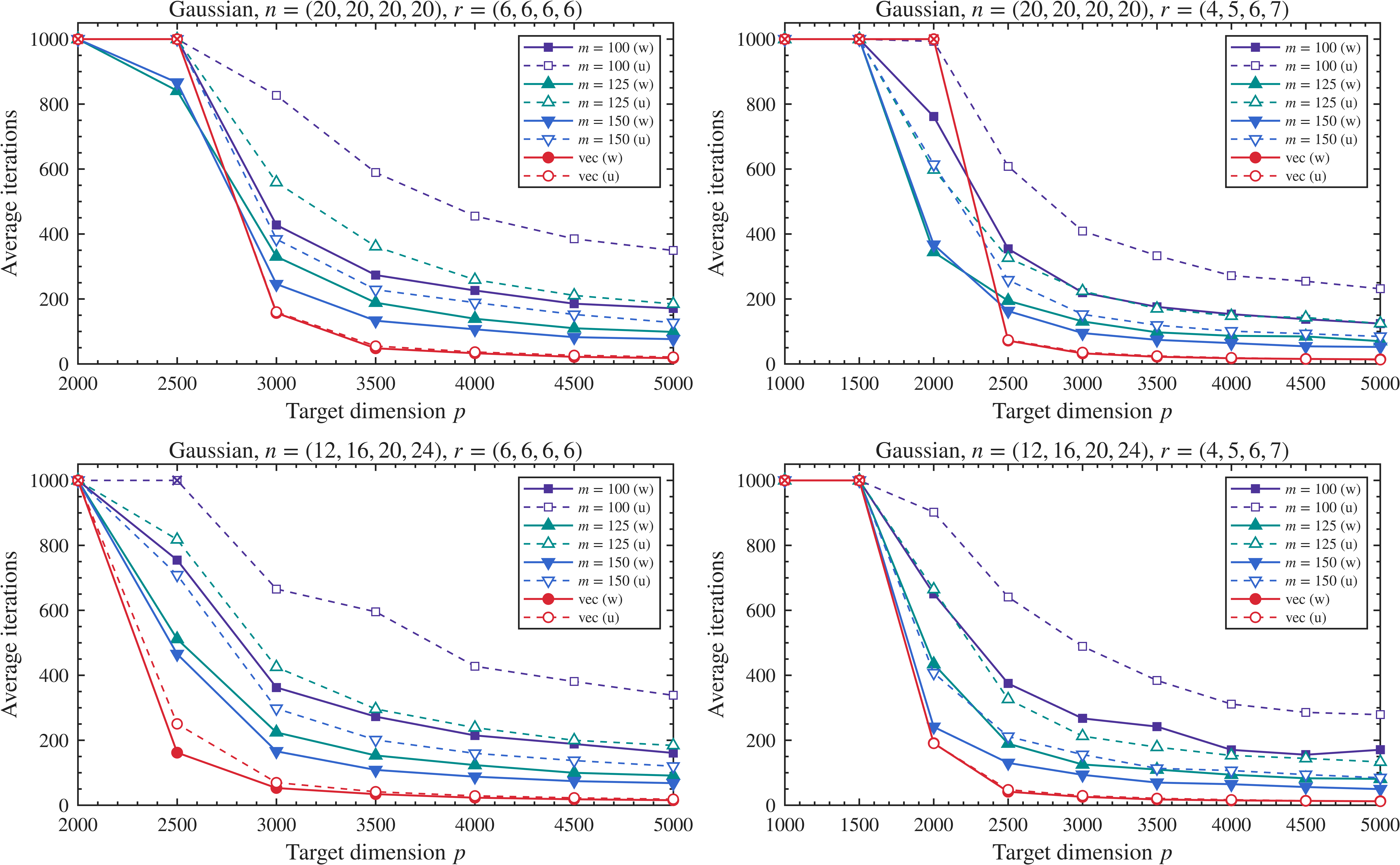}
	\caption{Average iteration numbers for Gaussian measurements. Values close to 1000 correspond to parameter regimes in which many trials do not reach the prescribed accuracy.}
	\label{fig:gaussian_iters}
\end{figure}

Figures~\ref{fig:sors_success} and \ref{fig:sors_iters} show the results
for SORS measurements. In the balanced-rank cases, the modewise methods
with $m=200$ or $250$ achieve nearly full recovery at target dimensions
comparable to, and sometimes smaller than, those required by the
vectorized methods. For the nonuniform rank $(5,6,7,8)$, increasing $m$
from $150$ to $200$ or $250$ substantially improves both the recovery
rate and convergence. The weighted variant generally requires fewer
iterations near the recovery threshold.

\begin{figure}[!tbp]
	\centering
	\includegraphics[width=0.98\linewidth]{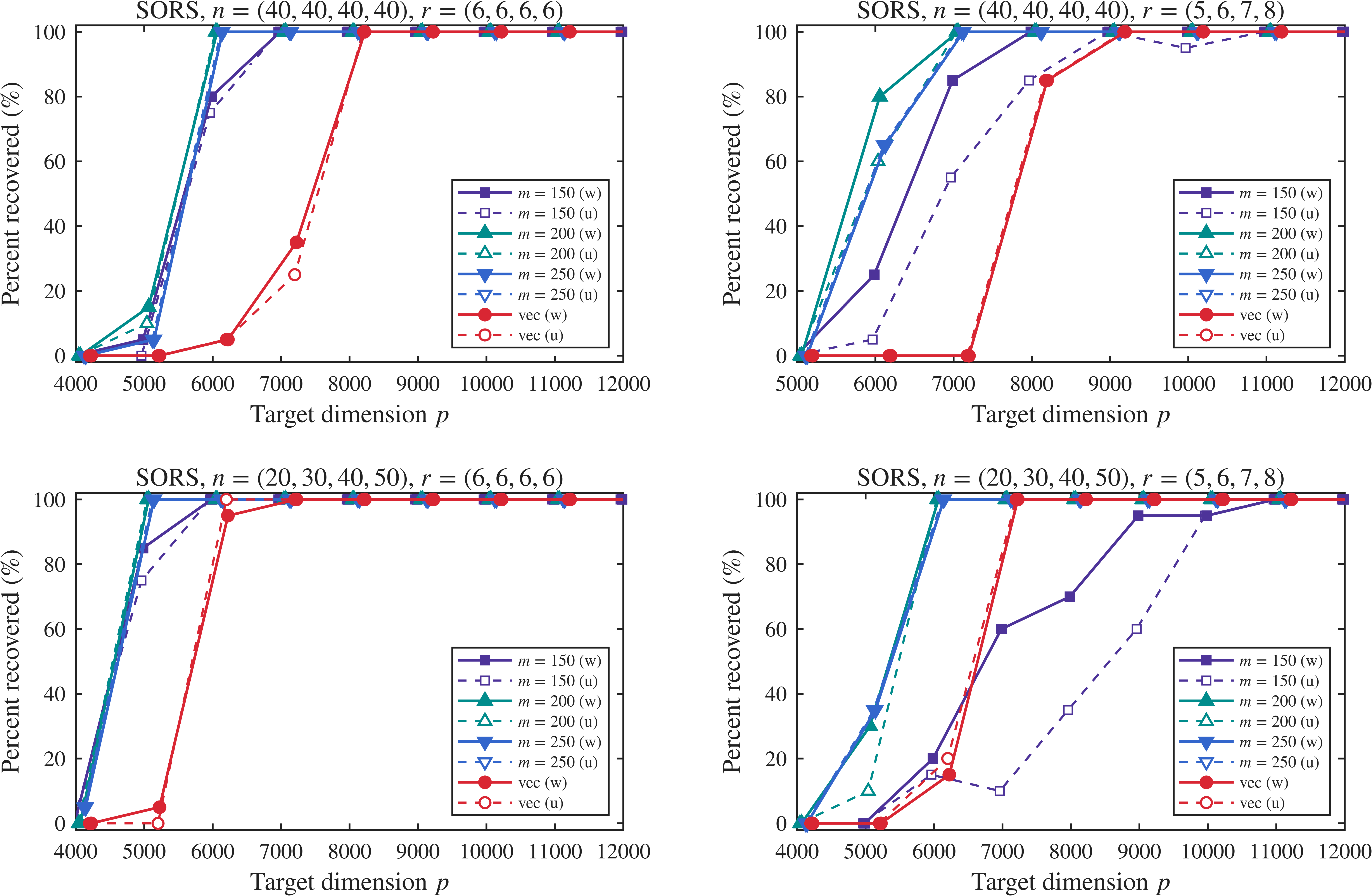}
	\caption{Successful recovery rates for SORS measurements under balanced and unbalanced tensor dimensions and multilinear ranks.}
	\label{fig:sors_success}
\end{figure}

\begin{figure}[!tbp]
	\centering
	\includegraphics[width=0.98\linewidth]{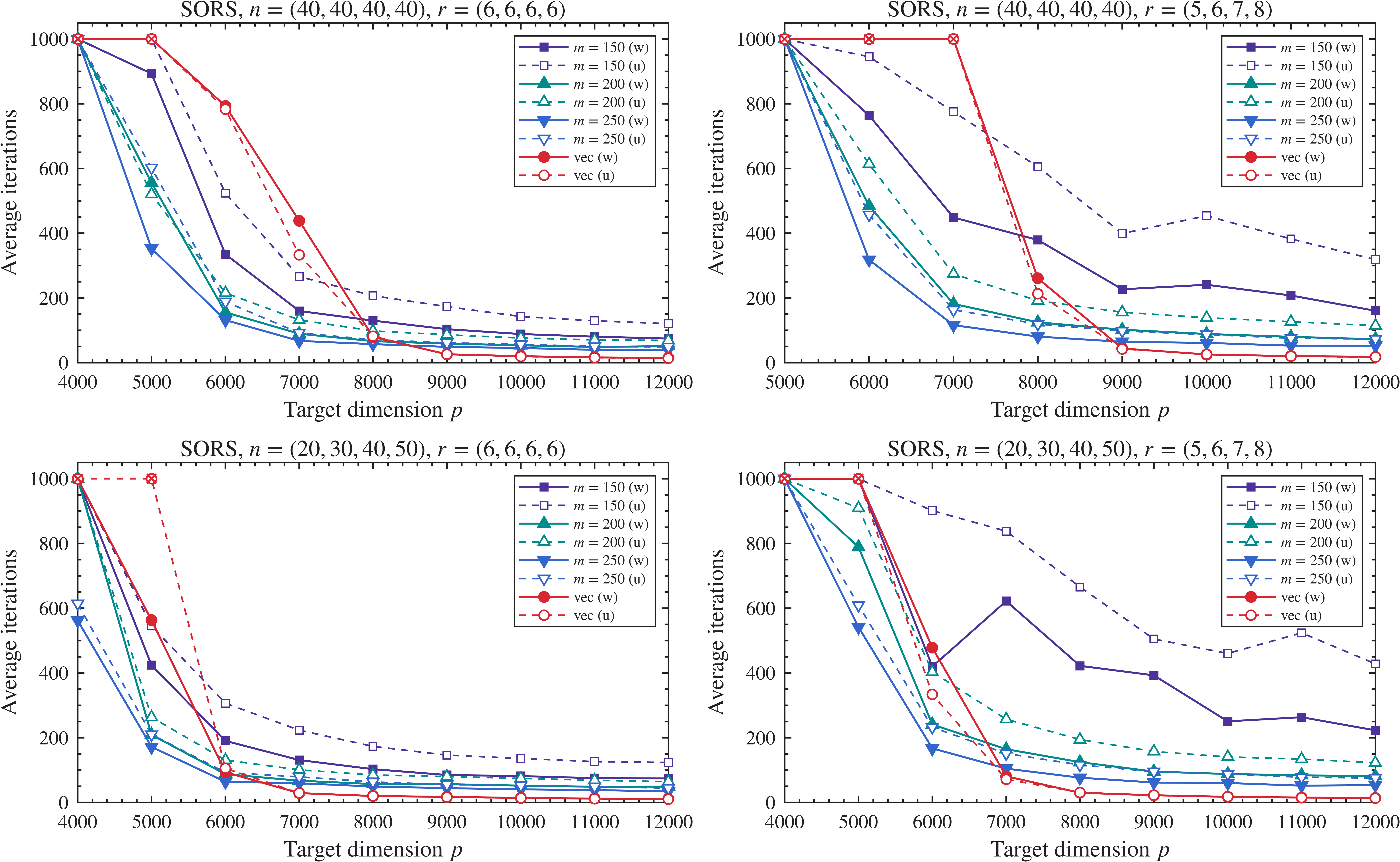}
	\caption{Average iteration numbers for SORS measurements. Increasing the intermediate dimension improves convergence, particularly in the unbalanced-rank experiments.}
	\label{fig:sors_iters}
\end{figure}

To further evaluate the computational efficiency,
Figure~\ref{fig:cpu_time} compares five methods: the proposed weighted
modewise RGD (weighted-mRGD), the unweighted modewise RGD (mRGD),
the vectorized RGD (vec-RGD), the modewise TIHT (TIHT), and the
vectorized TIHT (vec-TIHT), in terms of relative recovery error versus
CPU time. For Gaussian measurements, weighted-mRGD exhibits the fastest
convergence and reaches a relative error below $10^{-4}$ in approximately
$7$ seconds. Under SORS measurements, weighted-mRGD, mRGD, and vec-RGD
converge, whereas TIHT and vec-TIHT remain nearly stagnant. Among the
RGD methods, weighted-mRGD achieves the fastest error decay and the lowest
final error. These results show that the normalized weighting improves the
practical convergence efficiency of mRGD, particularly under SORS
measurements.

\begin{figure}[!tbp]
	\centering
	\includegraphics[width=0.98\linewidth]{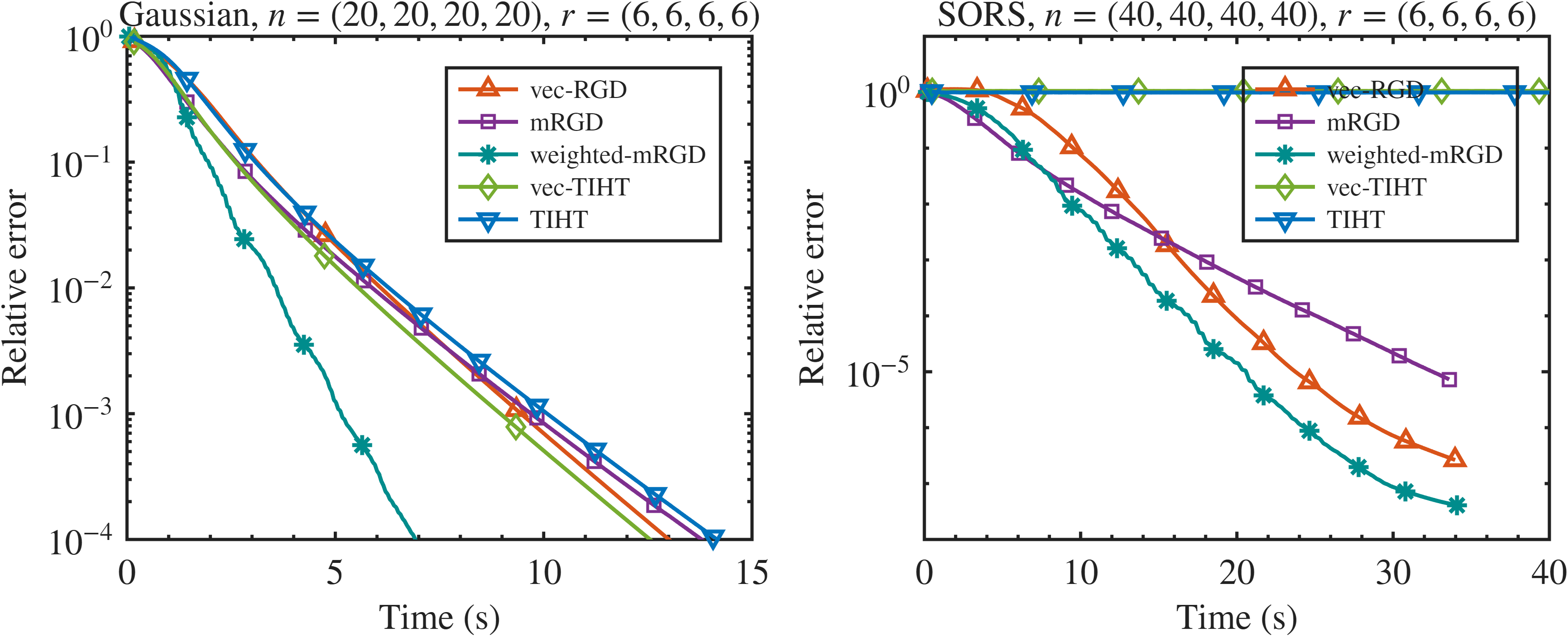}
	\caption{Relative error versus CPU time for weighted-mRGD, mRGD,
		vec-RGD, TIHT, and vec-TIHT. The left panel uses Gaussian measurements
		with $\boldsymbol{n}=(20,20,20,20)$ and
		$\boldsymbol{r}=(6,6,6,6)$, while the right panel uses SORS measurements
		with $\boldsymbol{n}=(40,40,40,40)$ and
		$\boldsymbol{r}=(6,6,6,6)$.}
	\label{fig:cpu_time}
\end{figure}

Figure~\ref{fig:method_comparison} compares weighted-mRGD, mRGD,
vec-RGD, TIHT, and vec-TIHT under Gaussian and SORS measurements.
For Gaussian measurements, all methods achieve accurate recovery when
$p$ is sufficiently large, while weighted-mRGD generally requires fewer
iterations near the recovery threshold. The difference is more evident
for SORS measurements. Both weighted-mRGD and mRGD reach relative errors
around $10^{-2}$ from $p\approx6000$, whereas vec-RGD requires a larger
target dimension. In contrast, TIHT and vec-TIHT remain inaccurate and
often reach the iteration cap over a large part of the tested range.
These results show that combining modewise measurements with RGD improves
the recovery efficiency, while the normalized weighting further
accelerates the convergence of mRGD.

\begin{figure}[!tbp]
	\centering
	\includegraphics[
	width=0.98\linewidth,
	keepaspectratio
	]{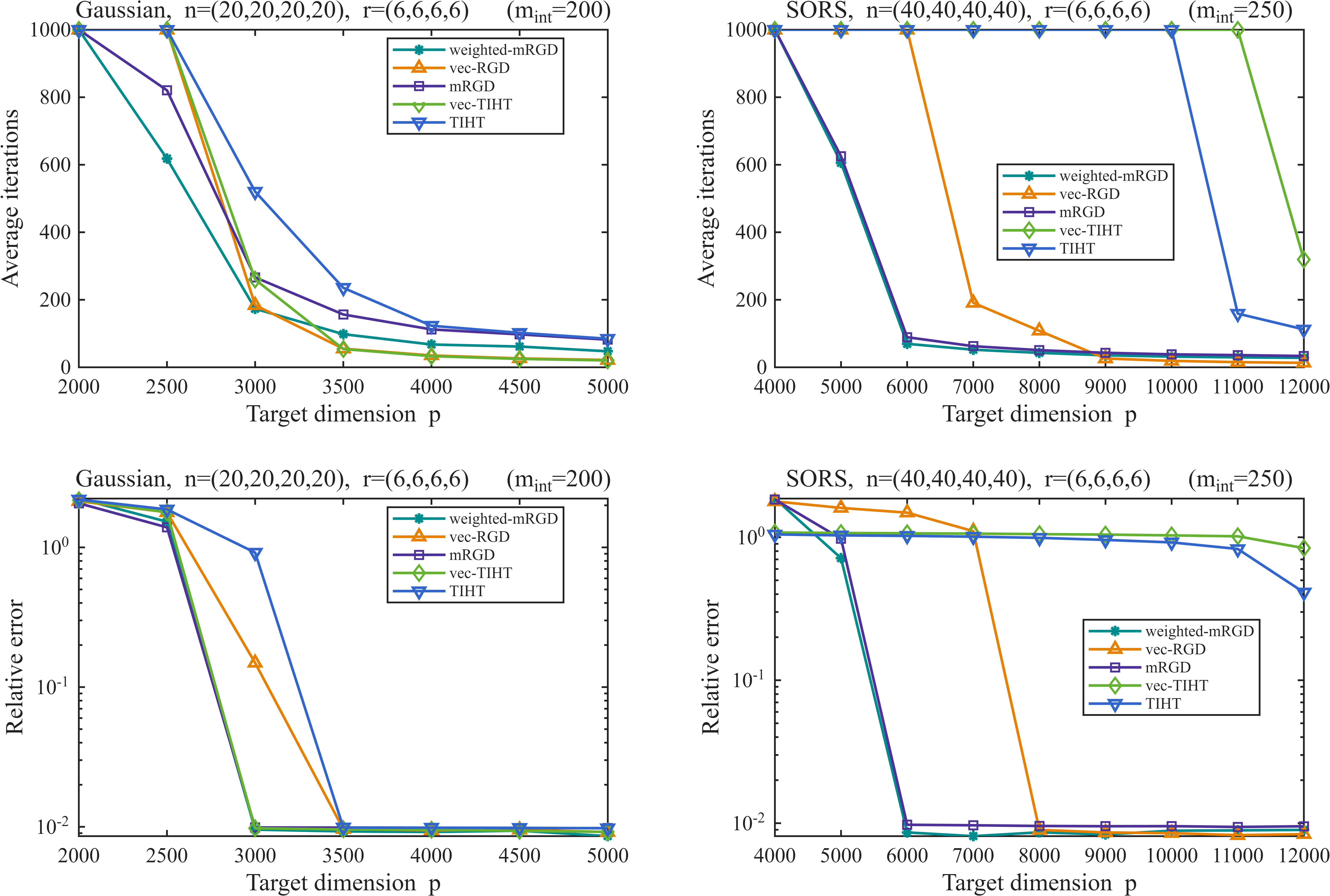}
	\caption{Comparison of weighted and unweighted RGD and TIHT
		methods under Gaussian and SORS measurements.
		The top and bottom rows report the average iteration numbers
		and final relative errors, respectively.
		The left column corresponds to Gaussian measurements with
		$\boldsymbol{n}=(20,20,20,20)$,
		$\boldsymbol{r}=(6,6,6,6)$, and
		$m_{\mathrm{int}}=200$, while the right column corresponds to
		SORS measurements with
		$\boldsymbol{n}=(40,40,40,40)$,
		$\boldsymbol{r}=(6,6,6,6)$, and
		$m_{\mathrm{int}}=250$.}
	\label{fig:method_comparison}
\end{figure}

Overall, the experiments confirm that the proposed method retains the recovery capability of vectorized RGD while using structured modewise measurements. The normalized adaptive weighting is most effective near the recovery threshold, where it generally improves the iteration count and, in several difficult settings, the empirical success rate.

\FloatBarrier
\section{Conclusion and Future Work}
\label{sec:conclusion}

In this paper, we combined modewise measurements with a normalized
block-weighted Riemannian gradient framework for low-multilinear-rank
tensor recovery. The proposed weighting adjusts the relative contributions
of the tangent-gradient components while preserving the underlying
multilinear-rank structure. Numerical results demonstrate improved
convergence efficiency, particularly near the recovery threshold and under
structured SORS measurements.
Future work will further explore the interaction between structured
modewise measurements and block-weighted Riemannian optimization, with
the aim of developing more efficient and scalable methods for large-scale
tensor recovery.

\appendix

\section*{Appendix: Complete Convergence Proofs}
\label{app:complete_convergence_proofs}
\addcontentsline{toc}{section}{Appendix: Complete Convergence Proofs}
\label{app:complete_convergence_proofs}

This appendix supplies the geometric and analytic estimates used in Section~\ref{sec:recovery_guarantee} and gives the complete proof of Theorem~\ref{thm:weighted_local_linear_convergence}.

For each mode $i\in[d]$, we introduce the modewise orthogonal projectors $\bm{\mathscr P}_{V_l^{(i)}}^{(i)}$ and $\bm{\mathscr P}_{U^{(i)}}^{(i)}$, which project the mode-$i$ fibers of any ambient tensor onto the column spaces of the current factor $V_l^{(i)}$ and the true factor $U^{(i)}$, respectively. They are equivalent to multiplying a tensor in mode $i$ by $V_l^{(i)}(V_l^{(i)})^T$ and $U^{(i)}(U^{(i)})^T$.

\begin{lemma}[Perturbation of the mode subspaces]
	\label{lem:app_mode_projector_perturbation}
	For every $i\in[d]$,
	\begin{equation*}
		\left\| \bm{\mathscr P}_{V_l^{(i)}}^{(i)} - \bm{\mathscr P}_{U^{(i)}}^{(i)} \right\|
		\le \frac{\|\mathcal X_l-\mathcal T\|_F}{\sigma_{r_i}(\mathcal T_{(i)})}.
	\end{equation*}
\end{lemma}

\begin{proof}
	By the definition of the modewise projectors and the invariance of the
	Frobenius norm under matricization,
	$\|\bm{\mathscr P}_{V_l^{(i)}}^{(i)}-\bm{\mathscr P}_{U^{(i)}}^{(i)}\|
	=
	\|V_l^{(i)}(V_l^{(i)})^T-U^{(i)}(U^{(i)})^T\|_2$.
	Since $V_l^{(i)}$ and $U^{(i)}$ have orthonormal columns and span
	subspaces of the same dimension $r_i$, the standard identity for
	equal-rank orthogonal projectors gives
	$\|V_l^{(i)}(V_l^{(i)})^T-U^{(i)}(U^{(i)})^T\|_2
	=
	\|\bm{\mathscr{P}}_{V_l^{(i)\perp}}U^{(i)}\|_2$.
	
	Let $\bar U^{(i)}$ denote the Kronecker product of the true factor
	matrices excluding the $i$-th mode, defined analogously to
	$\bar V_l^{(i)}$. Then
	$\mathcal T_{(i)}
	=
	U^{(i)}\mathcal B_{(i)}(\bar U^{(i)})^T$.
	Since $\operatorname{mulrank}(\mathcal T)=\boldsymbol r$,
	$\mathcal B_{(i)}$ has full row rank. Moreover, because
	$U^{(i)}$ and $\bar U^{(i)}$ have orthonormal columns,
	$\mathcal T_{(i)}$ and $\mathcal B_{(i)}$ have the same nonzero
	singular values. Hence
	$\|\mathcal B_{(i)}^\dagger\|_2
	=
	1/\sigma_{r_i}(\mathcal T_{(i)})$, and
	$U^{(i)}
	=
	\mathcal T_{(i)}\bar U^{(i)}\mathcal B_{(i)}^\dagger$.
	
	Since the column space of $(\mathcal X_l)_{(i)}$ is contained in
	$\operatorname{span}(V_l^{(i)})$,
	$\bm{\mathscr{P}}_{V_l^{(i)\perp}}(\mathcal X_l)_{(i)}=0$. Therefore,
	\begin{align*}
		\|\bm{\mathscr{P}}_{V_l^{(i)\perp}} U^{(i)}\|_2
		&=
		\left\|
		\bm{\mathscr{P}}_{V_l^{(i)\perp}}
		(\mathcal T-\mathcal X_l)_{(i)}
		\bar U^{(i)}
		\mathcal B_{(i)}^\dagger
		\right\|_2
		\nonumber\\
		&\le
		\|\bm{\mathscr{P}}_{V_l^{(i)\perp}}\|_2
		\|(\mathcal T-\mathcal X_l)_{(i)}\|_F
		\|\bar U^{(i)}\|_2
		\|\mathcal B_{(i)}^\dagger\|_2
		\le
		\frac{\|\mathcal X_l-\mathcal T\|_F}
		{\sigma_{r_i}(\mathcal T_{(i)})}.
	\end{align*}
\end{proof}

For later use, we define the orthogonal projector associated with the row space of the $i$-th factor tangent block by
\begin{equation}
	\left(
	\bm{\mathscr P}_{\mathcal B_l,\{V_l^{(j)}\}_{j\ne i}}^{(i)}\mathcal Y
	\right)_{(i)}
	=
	\mathcal Y_{(i)}
	\bar V_l^{(i)}
	(\mathcal B_l)_{(i)}^\dagger
	(\mathcal B_l)_{(i)}
	(\bar V_l^{(i)})^T.
	\label{eq:app_factor_row_projector}
\end{equation}
Because $(\mathcal B_l)_{(i)}^\dagger(\mathcal B_l)_{(i)}$ is the orthogonal projector onto the row space of $(\mathcal B_l)_{(i)}$ and $\bar V_l^{(i)}$ has orthonormal columns, \eqref{eq:app_factor_row_projector} indeed defines an orthogonal projector.

The closed-form expressions imply the following decomposition of the canonical tangent-space projector:
\begin{equation}
	\bm{\mathscr{P}}_{\mathcal S_l}
	=
	\prod_{i=1}^d
	\bm{\mathscr P}_{V_l^{(i)}}^{(i)}
	+
	\sum_{i=1}^d
	\bm{\mathscr P}_{\mathcal B_l,\{V_l^{(j)}\}_{j\ne i}}^{(i)}
	\bm{\mathscr P}_{V_l^{(i)\perp}}^{(i)}.
	\label{eq:app_tangent_projector_decomposition}
\end{equation}
The projectors appearing in each product in \eqref{eq:app_tangent_projector_decomposition} commute because the mode-$i$ projector acts on the left side of the mode-$i$ matricization, whereas $\bm{\mathscr P}_{\mathcal B_l,\{V_l^{(j)}\}_{j\ne i}}^{(i)}$ acts on the corresponding right side.

\begin{lemma}[Second-order tangent-space mismatch]
	\label{lem:app_tangent_mismatch}
	The orthogonal tangent-space projector satisfies
	\begin{equation}
		\|(\bm{\mathscr I}-\bm{\mathscr{P}}_{\mathcal S_l})\mathcal T\|_F
		\le
		\frac{2^d-1}{\min_{1\le i\le d}\sigma_{r_i}(\mathcal T_{(i)})}
		\|\mathcal X_l-\mathcal T\|_F^2.
		\label{eq:tangent_mismatch_second_order}
	\end{equation}
\end{lemma}

\begin{proof}
	Expanding
	$\bm{\mathscr I}=\prod_{i=1}^d
	(\bm{\mathscr P}_{V_l^{(i)}}^{(i)}
	+\bm{\mathscr P}_{V_l^{(i)\perp}}^{(i)})$
	and subtracting \eqref{eq:app_tangent_projector_decomposition}
	decomposes $\bm{\mathscr I}-\bm{\mathscr{P}}_{\mathcal S_l}$ into $d$ singleton terms
	($|\Lambda|=1$) and $2^d-1-d$ higher-order terms
	($|\Lambda|\ge2$):
	\begin{equation}
		\bm{\mathscr I}-\bm{\mathscr{P}}_{\mathcal S_l}
		=
		\sum_{i=1}^d
		\underbrace{\left(\prod_{j\ne i}\bm{\mathscr P}_{V_l^{(j)}}^{(j)}-\bm{\mathscr P}_{\mathcal B_l,\{V_l^{(j)}\}_{j\ne i}}^{(i)}\right)}_{:=\bm{\mathscr E}_i}
		\bm{\mathscr P}_{V_l^{(i)\perp}}^{(i)}
		+
		\sum_{\substack{\Lambda\subseteq[d]\\|\Lambda|\ge2}}
		\underbrace{\left(\prod_{j\in\Lambda\setminus\{i\}}\bm{\mathscr P}_{V_l^{(j)\perp}}^{(j)}\right)\left(\prod_{j\notin\Lambda}\bm{\mathscr P}_{V_l^{(j)}}^{(j)}\right)}_{:=\widetilde{\bm{\mathscr E}}_\Lambda}
		\bm{\mathscr P}_{V_l^{(i)\perp}}^{(i)}.
		\label{eq:app_normal_projector_expansion}
	\end{equation}
	where, for each higher-order term, $i$ denotes an arbitrary index
	chosen from $\Lambda$.
	
	Using the identity
	$\bm{\mathscr P}_{U^{(i)}}^{(i)}\mathcal T=\mathcal T$, we have
	$\bm{\mathscr P}_{V_l^{(i)\perp}}^{(i)}\mathcal T
	=
	(\bm{\mathscr P}_{U^{(i)}}^{(i)}
	-\bm{\mathscr P}_{V_l^{(i)}}^{(i)})\mathcal T$
	for every $i\in[d]$.
	
	For the singleton terms, both projectors defining $\bm{\mathscr E}_i$
	act as the identity on $\mathcal X_l$, and hence
	$\bm{\mathscr E}_i\mathcal X_l=0$. Moreover, the range of
	$\bm{\mathscr P}_{\mathcal B_l,\{V_l^{(j)}\}_{j\ne i}}^{(i)}$
	is contained in the range of
	$\prod_{j\ne i}\bm{\mathscr P}_{V_l^{(j)}}^{(j)}$.
	Therefore, $\bm{\mathscr E}_i$ is the difference of two nested
	orthogonal projectors and satisfies $\|\bm{\mathscr E}_i\|\le1$.
	Since $\bm{\mathscr E}_i$ commutes with the mode-$i$ projectors,
	\begin{equation*}
		\bm{\mathscr E}_i \bm{\mathscr P}_{V_l^{(i)\perp}}^{(i)}\mathcal T
		=
		(\bm{\mathscr P}_{U^{(i)}}^{(i)}-\bm{\mathscr P}_{V_l^{(i)}}^{(i)})
		\bm{\mathscr E}_i\mathcal T
		=
		(\bm{\mathscr P}_{U^{(i)}}^{(i)}-\bm{\mathscr P}_{V_l^{(i)}}^{(i)})
		\bm{\mathscr E}_i(\mathcal T-\mathcal X_l).
	\end{equation*}
	Hence, Lemma~\ref{lem:app_mode_projector_perturbation} gives
	\[
	\left\|
	\bm{\mathscr E}_i
	\bm{\mathscr P}_{V_l^{(i)\perp}}^{(i)}\mathcal T
	\right\|_F
	\le
	\frac{\|\mathcal T-\mathcal X_l\|_F^2}
	{\sigma_{r_i}(\mathcal T_{(i)})}.
	\]
	
	For each higher-order term,
	$\widetilde{\bm{\mathscr E}}_\Lambda$ contains at least one
	mode-$j$ orthogonal complement projector with $j\ne i$, so that
	$\widetilde{\bm{\mathscr E}}_\Lambda\mathcal X_l=0$.
	Furthermore, since it is a product of commuting orthogonal
	projectors, $\|\widetilde{\bm{\mathscr E}}_\Lambda\|\le1$, and it
	commutes with the mode-$i$ projectors. The same argument therefore
	yields
	\[
	\left\|
	\widetilde{\bm{\mathscr E}}_\Lambda
	\bm{\mathscr P}_{V_l^{(i)\perp}}^{(i)}\mathcal T
	\right\|_F
	\le
	\frac{\|\mathcal T-\mathcal X_l\|_F^2}
	{\sigma_{r_i}(\mathcal T_{(i)})}.
	\]
	There are $d$ singleton terms and $2^d-1-d$ higher-order terms
	in \eqref{eq:app_normal_projector_expansion}. Summing their
	Frobenius norms and uniformly bounding the denominators by
	$\min_{1\le i\le d}\sigma_{r_i}(\mathcal T_{(i)})$ yields
	\eqref{eq:tangent_mismatch_second_order}.
\end{proof}

\begin{lemma}[Restricted tangent-space operator bounds]
	\label{lem:app_restricted_trip_operator}
	Assume that $\bm{\mathscr L}$ satisfies
	$\operatorname{TRIP}(\delta_{2\boldsymbol r},2\boldsymbol r)$. Then
	\begin{equation}
		\left\|
		\bm{\mathscr{P}}_{\mathcal S_l}-
		\bm{\mathscr{P}}_{\mathcal S_l}\bm{\mathscr L}^*\bm{\mathscr L} \bm{\mathscr{P}}_{\mathcal S_l}
		\right\|
		\le\delta_{2\boldsymbol r},
		\qquad
		\left\|
		\bm{\mathscr{P}}_{\mathcal S_l}\bm{\mathscr L}^*\bm{\mathscr L} \bm{\mathscr{P}}_{\mathcal S_l}
		\right\|
		\le1+\delta_{2\boldsymbol r}.
		\label{eq:app_restricted_trip_operator_bounds}
	\end{equation}
\end{lemma}

\begin{proof}
	The first operator in \eqref{eq:app_restricted_trip_operator_bounds}
	is self-adjoint and vanishes on $\mathcal S_l^\perp$, so its operator
	norm is determined by its restriction to $\mathcal S_l$. For any
	tensor $\mathcal Z\in\mathcal S_l$ with $\|\mathcal Z\|_F=1$, we have $\bm{\mathscr{P}}_{\mathcal S_l}\mathcal Z=\mathcal Z$ and 
	$\operatorname{mulrank}(\mathcal Z)\preceq 2\boldsymbol r$, and hence
	the TRIP implies $|\|\mathcal Z\|_F^2-\|\bm{\mathscr L}(\mathcal Z)\|_2^2|\le\delta_{2\boldsymbol r}$.
	By the linearity of the inner product, the definition of the adjoint operator $\bm{\mathscr L}^*$, and the property $\bm{\mathscr{P}}_{\mathcal S_l}\mathcal Z=\mathcal Z$, we can expand the quadratic form as:
	\begin{align*}
		\left\langle \left( \bm{\mathscr{P}}_{\mathcal S_l}-\bm{\mathscr{P}}_{\mathcal S_l}\bm{\mathscr L}^*\bm{\mathscr L} \bm{\mathscr{P}}_{\mathcal S_l} \right)\mathcal Z, \mathcal Z \right\rangle_F 
		&= \langle \bm{\mathscr{P}}_{\mathcal S_l}\mathcal Z, \mathcal Z \rangle_F - \langle \bm{\mathscr{P}}_{\mathcal S_l}\bm{\mathscr L}^*\bm{\mathscr L} \bm{\mathscr{P}}_{\mathcal S_l}\mathcal Z, \mathcal Z \rangle_F \\
		&= \langle \mathcal Z, \mathcal Z \rangle_F - \langle \bm{\mathscr L}^*\bm{\mathscr L} \mathcal Z, \mathcal Z \rangle_F = \|\mathcal Z\|_F^2 - \|\bm{\mathscr L}(\mathcal Z)\|_2^2
		\le\delta_{2\boldsymbol r}.
	\end{align*}
	Thus, the first bound follows by taking the supremum of the absolute value over unit tensors in $\mathcal S_l$.
	The second operator is positive semidefinite and also vanishes on $\mathcal S_l^\perp$. By a similar Rayleigh quotient argument, the upper TRIP inequality guarantees $\langle \bm{\mathscr{P}}_{\mathcal S_l}\bm{\mathscr L}^*\bm{\mathscr L} \bm{\mathscr{P}}_{\mathcal S_l}\mathcal Z, \mathcal Z \rangle_F = \|\bm{\mathscr L}(\mathcal Z)\|_2^2 \le 1+\delta_{2\boldsymbol r}$ for any unit tensor $\mathcal Z \in \mathcal S_l$, yielding $\|\bm{\mathscr{P}}_{\mathcal S_l}\bm{\mathscr L}^*\bm{\mathscr L} \bm{\mathscr{P}}_{\mathcal S_l}\| \le 1+\delta_{2\boldsymbol r}$.
\end{proof}

\begin{lemma}[Rank and TRIP bound for the normal residual]
	\label{lem:app_normal_residual_cross_term}
	The tensor $(\bm{\mathscr I}-\bm{\mathscr{P}}_{\mathcal S_l})\mathcal T$ has multilinear rank at most $2\boldsymbol r$. Consequently, under $\operatorname{TRIP}(\delta_{2\boldsymbol r},2\boldsymbol r)$, it holds that
	\begin{equation*}
		\|\bm{\mathscr{P}}_{\mathcal S_l}\bm{\mathscr L}^*\bm{\mathscr L}
		(\bm{\mathscr I}-\bm{\mathscr{P}}_{\mathcal S_l})\mathcal T\|_F
		\le
		(1+\delta_{2\boldsymbol r})
		\|(\bm{\mathscr I}-\bm{\mathscr{P}}_{\mathcal S_l})\mathcal T\|_F.
	\end{equation*}
\end{lemma}

\begin{proof}
	By the tangent projection formula \eqref{eq:factor_closed_form}, the
	mode-$k$ column space of $\bm{\mathscr{P}}_{\mathcal S_l}\mathcal T$ is contained
	in the sum of $\operatorname{span}(V_l^{(k)})$ and
	$\operatorname{col}(\bm{\mathscr P}_{V_l^{(k)\perp}}\mathcal T_{(k)})$.
	Since $\operatorname{col}(\mathcal T_{(k)})
	=\operatorname{span}(U^{(k)})$, we have
	$\operatorname{col}(\bm{\mathscr P}_{V_l^{(k)\perp}}\mathcal T_{(k)})
	\subseteq\operatorname{span}(U^{(k)},V_l^{(k)})$. Hence
	$\operatorname{col}((\bm{\mathscr{P}}_{\mathcal S_l}\mathcal T)_{(k)})
	\subseteq\operatorname{span}(U^{(k)},V_l^{(k)})$.
	Since the same inclusion holds for $\mathcal T_{(k)}$, it follows that
	$\operatorname{col}(((\bm{\mathscr I}-\bm{\mathscr{P}}_{\mathcal S_l})\mathcal T)_{(k)})
	\subseteq\operatorname{span}(U^{(k)},V_l^{(k)})$.
	Therefore,
	$\operatorname{mulrank}((\bm{\mathscr I}-\bm{\mathscr{P}}_{\mathcal S_l})\mathcal T)
	\preceq2\boldsymbol r$.
	
	For the cross-term bound, the self-adjointness of
	$\bm{\mathscr{P}}_{\mathcal S_l}$ and the duality of the Frobenius norm give
	\begin{equation*}
		\left\|\bm{\mathscr{P}}_{\mathcal S_l}\bm{\mathscr L}^*\bm{\mathscr L}(\bm{\mathscr I}-\bm{\mathscr{P}}_{\mathcal S_l})\mathcal T\right\|_F
		=
		\sup_{\substack{\mathcal Z\in\mathcal S_l\\ \|\mathcal Z\|_F=1}}
		\left|
		\left\langle
		\bm{\mathscr L}(\bm{\mathscr I}-\bm{\mathscr{P}}_{\mathcal S_l})\mathcal T,
		\bm{\mathscr L}(\mathcal Z)
		\right\rangle
		\right|
		\le
		(1+\delta_{2\boldsymbol r})\left\|(\bm{\mathscr I}-\bm{\mathscr{P}}_{\mathcal S_l})\mathcal T\right\|_F,
	\end{equation*}
	where the inequality follows from the Cauchy--Schwarz inequality and
	the upper TRIP bound applied separately to
	$(\bm{\mathscr I}-\bm{\mathscr{P}}_{\mathcal S_l})\mathcal T$ and $\mathcal Z$, both of which have
	multilinear rank at most $2\boldsymbol r$.
\end{proof}

We next give the complete proofs of the two adaptive perturbation estimates used in Section~\ref{subsec:main_local_convergence}.

\begin{lemma}[Operator bound for the weight perturbation]
	\label{lem:El_bound}
	Let $\rho_l:=\max_{0\le k\le d}|\omega_{l,k}-1|$. Then, for every ambient tensor $\mathcal Z$,
	\begin{equation}
		\|(\bm{\mathscr{P}}_l^{\boldsymbol\omega_l}-\bm{\mathscr{P}}_{\mathcal S_l})\mathcal Z\|_F
		\le
		\rho_l
		\|\bm{\mathscr{P}}_{\mathcal S_l}\mathcal Z\|_F.
		\label{eq:El_operator_bound}
	\end{equation}
	In particular,
	\begin{equation}
		\|(\bm{\mathscr{P}}_l^{\boldsymbol\omega_l}-\bm{\mathscr{P}}_{\mathcal S_l})\|_{\mathcal S_l\to\mathcal S_l}
		\le
		\rho_l.
		\label{eq:El_restricted_operator_norm}
	\end{equation}
\end{lemma}

\begin{proof}
	Since the ranges of $\bm{\Pi}_l^{(0)},\ldots,\bm{\Pi}_l^{(d)}$ are mutually orthogonal, for every ambient tensor $\mathcal Z$ we have
	\begin{equation}
		\left\|(\bm{\mathscr{P}}_l^{\boldsymbol\omega_l}-\bm{\mathscr{P}}_{\mathcal S_l})\mathcal Z\right\|_F^2
		=
		\sum_{k=0}^d(\omega_{l,k}-1)^2\|\bm{\Pi}_l^{(k)}\mathcal Z\|_F^2
		\le
		\rho_l^2\sum_{k=0}^d\|\bm{\Pi}_l^{(k)}\mathcal Z\|_F^2
		=
		\rho_l^2\|\bm{\mathscr{P}}_{\mathcal S_l}\mathcal Z\|_F^2.
		\label{eq:app_weight_perturbation_orthogonal_sum}
	\end{equation}
	Taking square roots in \eqref{eq:app_weight_perturbation_orthogonal_sum} gives \eqref{eq:El_operator_bound}. For $\mathcal Z\in\mathcal S_l$, we have $\bm{\mathscr{P}}_{\mathcal S_l}\mathcal Z=\mathcal Z$, and hence taking the supremum over nonzero tensors in $\mathcal S_l$ gives \eqref{eq:El_restricted_operator_norm}.
\end{proof}

\begin{lemma}[Adaptive-weight cross term]
	\label{lem:weighted_cross_term}
	Assume that $\bm{\mathscr L}$ satisfies
	$\operatorname{TRIP}(\delta_{2\boldsymbol r},2\boldsymbol r)$. Then
	\begin{equation}
		\begin{split}
			&\|(\bm{\mathscr{P}}_l^{\boldsymbol\omega_l}-\bm{\mathscr{P}}_{\mathcal S_l})\bm{\mathscr L}^*\bm{\mathscr L}
			(\mathcal X_l-\mathcal T)\|_F\\
			&\qquad\le
			\rho_l(1+\delta_{2\boldsymbol r})
			\left(
			\|\mathcal X_l-\mathcal T\|_F+
			\frac{2^d-1}{\min_{1\le i\le d}\sigma_{r_i}(\mathcal T_{(i)})}\|\mathcal X_l-\mathcal T\|_F^2
			\right).
		\end{split}
		\label{eq:weighted_cross_term_bound}
	\end{equation}
\end{lemma}

\begin{proof}
	Since $\mathcal X_l\in\mathcal S_l$, we have $\bm{\mathscr{P}}_{\mathcal S_l}\mathcal X_l=\mathcal X_l$ and therefore $\mathcal X_l-\mathcal T=\bm{\mathscr{P}}_{\mathcal S_l}(\mathcal X_l-\mathcal T)-(\bm{\mathscr I}-\bm{\mathscr{P}}_{\mathcal S_l})\mathcal T$. Applying Lemma~\ref{lem:El_bound} to $\bm{\mathscr L}^*\bm{\mathscr L}(\mathcal X_l-\mathcal T)$ and then using this decomposition gives
	\begin{align}
		&\left\|
		(\bm{\mathscr{P}}_l^{\boldsymbol\omega_l}-\bm{\mathscr{P}}_{\mathcal S_l})
		\bm{\mathscr L}^*\bm{\mathscr L}(\mathcal X_l-\mathcal T)
		\right\|_F
		\nonumber\\
		&\quad\le
		\rho_l
		\left\|
		\bm{\mathscr{P}}_{\mathcal S_l}\bm{\mathscr L}^*\bm{\mathscr L}
		\bm{\mathscr{P}}_{\mathcal S_l}(\mathcal X_l-\mathcal T)
		\right\|_F
		+
		\rho_l
		\left\|
		\bm{\mathscr{P}}_{\mathcal S_l}\bm{\mathscr L}^*\bm{\mathscr L}
		(\bm{\mathscr I}-\bm{\mathscr{P}}_{\mathcal S_l})\mathcal T
		\right\|_F
		\nonumber\\
		&\quad\le
		\rho_l(1+\delta_{2\boldsymbol r})
		\left(
		\|\mathcal X_l-\mathcal T\|_F+
		\|(\bm{\mathscr I}-\bm{\mathscr{P}}_{\mathcal S_l})\mathcal T\|_F
		\right).
		\label{eq:app_weighted_cross_term_chain}
	\end{align}
	In the last inequality of \eqref{eq:app_weighted_cross_term_chain}, we used Lemma~\ref{lem:app_restricted_trip_operator}, Lemma~\ref{lem:app_normal_residual_cross_term}, and $\|\bm{\mathscr{P}}_{\mathcal S_l}(\mathcal X_l-\mathcal T)\|_F\le\|\mathcal X_l-\mathcal T\|_F$. Applying the second-order estimate \eqref{eq:tangent_mismatch_second_order} to the remaining normal component proves \eqref{eq:weighted_cross_term_bound}.
\end{proof}

\begin{lemma}[Step-size bound for exact line search]
	\label{lem:weighted_stepsize_bound}
	Suppose $\rho_l := \max_{0\le k\le d}|\omega_{l,k}-1| < 1$ and $\bm{\mathscr L}$ satisfies $\operatorname{TRIP}(\delta_{2\boldsymbol r}, 2\boldsymbol{r})$. Then the step size $\alpha_l$ in \eqref{eq:exact_stepsize} satisfies
	\begin{equation}
		\frac{1}{(1+\delta_{2\boldsymbol r})(1+\rho_l)} \le \alpha_l \le \frac{1}{(1-\delta_{2\boldsymbol r})(1-\rho_l)}.
		\label{eq:weighted_stepsize_interval}
	\end{equation}
	Consequently,
	\begin{equation}
		|\alpha_l-1| \le \frac{1}{(1-\delta_{2\boldsymbol r})(1-\rho_l)} - 1.
		\label{eq:weighted_stepsize_deviation}
	\end{equation}
\end{lemma}

\begin{proof}
	Let $a_{l,k}:=\|\bm{\Pi}_l^{(k)}\mathcal G_l\|_F$. The mutual
	orthogonality of the tangent blocks yields $\langle\mathcal G_l,\mathcal Z_l\rangle_F=\sum_{k=0}^{d}\omega_{l,k}a_{l,k}^2$ and $\|\mathcal Z_l\|_F^2=\sum_{k=0}^{d}\omega_{l,k}^2a_{l,k}^2$.
	Since $1-\rho_l\le\omega_{l,k}\le1+\rho_l$, it follows that
	\begin{equation}
		(1-\rho_l)\langle\mathcal G_l,\mathcal Z_l\rangle_F
		\le
		\|\mathcal Z_l\|_F^2
		\le
		(1+\rho_l)\langle\mathcal G_l,\mathcal Z_l\rangle_F.
		\label{eq:weighted_inner_norm_bound}
	\end{equation}
	
	Moreover, $\mathcal Z_l\in\mathcal S_l$ implies
	$\operatorname{mulrank}(\mathcal Z_l)\preceq2\boldsymbol r$.
	Hence $\operatorname{TRIP}(\delta_{2\boldsymbol r},2\boldsymbol r)$ gives
	\begin{equation}
		(1-\delta_{2\boldsymbol r})\|\mathcal Z_l\|_F^2
		\le
		\|\bm{\mathscr L}(\mathcal Z_l)\|_2^2
		\le
		(1+\delta_{2\boldsymbol r})\|\mathcal Z_l\|_F^2.
		\label{eq:TRIP_weighted_direction}
	\end{equation}
	Combining \eqref{eq:weighted_inner_norm_bound} and
	\eqref{eq:TRIP_weighted_direction} with
	$\alpha_l
	=
	\frac{\langle\mathcal G_l,\mathcal Z_l\rangle_F}{\|\bm{\mathscr L}(\mathcal Z_l)\|_2^2}$
	yields \eqref{eq:weighted_stepsize_interval}.
	
	The upper endpoint in \eqref{eq:weighted_stepsize_interval} gives
	the larger deviation from one, which proves \eqref{eq:weighted_stepsize_deviation}.
\end{proof}

\begin{lemma}[Uniform control of the adaptive-weight deviation]
	\label{lem:app_weight_deviation_control}
	Suppose that $\max_{0\le k\le d}\|\bm{\Pi}_l^{(k)}\mathcal G_l\|_F\le\Gamma$,
	and define $S_{\Gamma}:=\sqrt{\Gamma^2+\eta^2}$. Then
	\begin{equation}
		\rho_l
		:=
		\max_{0\le k\le d}
		|\omega_{l,k}-1|
		\le
		\rho_\eta(\Gamma)
		:=
		\frac{
			d(S_{\Gamma}-\eta)
		}{
			S_{\Gamma}+d\eta
		}.
		\label{eq:rho_eta_bound}
	\end{equation}
\end{lemma}

\begin{proof}
	By the assumed block-gradient bound and the definition of $s_{l,k}$,
	we have $\eta\le s_{l,k}\le S_{\Gamma}$ for all $k=0,\ldots,d$.
	Hence, for each $k$,
	\begin{equation*}
		\frac{(d+1)\eta}{\eta+dS_{\Gamma}}
		\le
		\omega_{l,k}
		\le
		\frac{(d+1)S_{\Gamma}}{S_{\Gamma}+d\eta}.
	\end{equation*}
	Therefore, $\omega_{l,k}-1\le\frac{d(S_{\Gamma}-\eta)}{S_{\Gamma}+d\eta}$, $1-\omega_{l,k}\le\frac{d(S_{\Gamma}-\eta)}{\eta+dS_{\Gamma}}$.
	Since $S_{\Gamma}\ge\eta$ and $d\ge1$,
	$S_{\Gamma}+d\eta\le\eta+dS_{\Gamma}$. Thus, $|\omega_{l,k}-1|\le\frac{d(S_{\Gamma}-\eta)}{S_{\Gamma}+d\eta}=\rho_\eta(\Gamma)$.
	Taking the maximum over $k=0,\ldots,d$ proves
	\eqref{eq:rho_eta_bound}.
\end{proof}

Under the
$\operatorname{TRIP}(\delta_{2\boldsymbol r},2\boldsymbol r)$
assumption, whenever
$\|\mathcal X_l-\mathcal T\|_F\le R$, the block-gradient
estimate gives
\begin{align}
	\|\bm{\Pi}_l^{(k)}\mathcal G_l\|_F
	&=
	\sup_{\substack{\mathcal Z\in\operatorname{range}(\bm{\Pi}_l^{(k)})\\ \|\mathcal Z\|_F=1}}
	\left|
	\left\langle
	\mathcal Z,
	\bm{\mathscr L}^*\bm{\mathscr L}(\mathcal X_l-\mathcal T)
	\right\rangle_F
	\right|
	=
	\sup_{\substack{\mathcal Z\in\operatorname{range}(\bm{\Pi}_l^{(k)})\\ \|\mathcal Z\|_F=1}}
	\left|
	\left\langle
	\bm{\mathscr L}(\mathcal Z),
	\bm{\mathscr L}(\mathcal X_l-\mathcal T)
	\right\rangle
	\right|
	\nonumber\\
	&\le
	\sup_{\substack{\mathcal Z\in\operatorname{range}(\bm{\Pi}_l^{(k)})\\ \|\mathcal Z\|_F=1}}
	\|\bm{\mathscr L}(\mathcal Z)\|_2
	\|\bm{\mathscr L}(\mathcal X_l-\mathcal T)\|_2
	\le
	(1+\delta_{2\boldsymbol r})
	\|\mathcal X_l-\mathcal T\|_F.
	\label{eq:trip_block_gradient_bound}
\end{align}

Applying Lemma~\ref{lem:app_weight_deviation_control} with
$\Gamma=(1+\delta_{2\boldsymbol r})R$ yields $\|\mathcal X_l-\mathcal T\|_F\le R\Longrightarrow\rho_l\le\rho_\eta\!\left((1+\delta_{2\boldsymbol r})R\right)$.
If $\eta$ is chosen such that
$\rho_\eta((1+\delta_{2\boldsymbol r})R)\le\rho_\star$, then
\begin{equation}
	\|\mathcal X_l-\mathcal T\|_F\le R
	\quad\Longrightarrow\quad
	\rho_l\le\rho_\star.
	\label{eq:local_error_implies_rho}
\end{equation}

\subsection{Complete proof of Theorem~\ref{thm:weighted_local_linear_convergence}}
\label{app:complete_main_theorem_proof}

\begin{proof}
	Fix $l \ge 0$ such that $\|\mathcal X_l-\mathcal T\|_F\le R$. By \eqref{eq:local_error_implies_rho}, $\rho_l\le\rho_\star<1$.
	If $\mathcal Z_l=0$, then $\bm{\mathscr{P}}_{\mathcal S_l}\mathcal G_l=0$. Lemmas~\ref{lem:app_tangent_mismatch} and \ref{lem:app_restricted_trip_operator} then force $\|\mathcal X_l-\mathcal T\|_F = 0$ (since $\gamma_R < 1$), meaning $\mathcal X_l=\mathcal T$, and the theorem holds.
	Assume $\mathcal Z_l\neq 0$. Since $\mathcal Z_l\in\mathcal S_l$, $\operatorname{mulrank}(\mathcal Z_l)\preceq 2\boldsymbol{r}$. The lower TRIP bound guarantees $\|\bm{\mathscr L}(\mathcal Z_l)\|_2^2 \ge (1-\delta_{2\boldsymbol r})\|\mathcal Z_l\|_F^2 > 0$, ensuring $\alpha_l$ is well-defined.
	Since $\mathcal X_{l+1}=\bm{\mathscr H}_{\boldsymbol r}(\mathcal Y_l)$,
	\eqref{eq:hosvd_quasi_optimality} and
	$\mathcal T\in\mathbb M_{\boldsymbol r}$ imply
	\begin{equation*}
		\|\mathcal X_{l+1}-\mathcal Y_l\|_F
		\le
		\sqrt d\,
		\|\mathcal Y_l-\bm{\mathscr P}_{\mathbb M_{\boldsymbol r}}(\mathcal Y_l)\|_F
		\le
		\sqrt d\,
		\|\mathcal Y_l-\mathcal T\|_F.
	\end{equation*}
	Therefore, by the triangle inequality,
	\begin{equation*}
		\|\mathcal X_{l+1}-\mathcal T\|_F
		\le
		\|\mathcal X_{l+1}-\mathcal Y_l\|_F
		+
		\|\mathcal Y_l-\mathcal T\|_F
		\le
		(\sqrt d+1)\|\mathcal Y_l-\mathcal T\|_F.
	\end{equation*} 
	Substituting $\mathcal Y_l = \mathcal X_l - \alpha_l \bm{\mathscr{P}}_l^{\boldsymbol\omega_l} \bm{\mathscr L}^*\bm{\mathscr L}(\mathcal X_l-\mathcal T)$ and utilizing $\mathcal X_l\in\mathcal S_l$, the error decomposes as
	\begin{align*}
		\|\mathcal X_{l+1}-\mathcal T\|_F
		&\le
		(\sqrt d+1)
		\left\|
		(\mathcal X_l-\mathcal T)
		-\alpha_l \bm{\mathscr{P}}_l^{\boldsymbol\omega_l}
		\bm{\mathscr L}^*\bm{\mathscr L}(\mathcal X_l-\mathcal T)
		\right\|_F
		\nonumber\\
		&=
		(\sqrt d+1)
		\left\|
		(\mathcal X_l-\mathcal T)
		-\alpha_l \bm{\mathscr{P}}_{\mathcal S_l}\bm{\mathscr L}^*\bm{\mathscr L}(\mathcal X_l-\mathcal T)
		-\alpha_l(\bm{\mathscr{P}}_l^{\boldsymbol\omega_l}-\bm{\mathscr{P}}_{\mathcal S_l})
		\bm{\mathscr L}^*\bm{\mathscr L}(\mathcal X_l-\mathcal T)
		\right\|_F
		\nonumber\\
		&\le
		(\sqrt d+1)\Big(
		\underbrace{
			\left\|
			\left(\bm{\mathscr{P}}_{\mathcal S_l}
			-\alpha_l\bm{\mathscr{P}}_{\mathcal S_l}\bm{\mathscr L}^*\bm{\mathscr L} \bm{\mathscr{P}}_{\mathcal S_l}\right)
			(\mathcal X_l-\mathcal T)
			\right\|_F
		}_{I_1}
		+
		\underbrace{
			\left\|
			(\bm{\mathscr I}-\bm{\mathscr{P}}_{\mathcal S_l})(\mathcal X_l-\mathcal T)
			\right\|_F
		}_{I_2}
		\nonumber\\
		&\qquad\qquad
		+
		\underbrace{
			\alpha_l
			\left\|
			\bm{\mathscr{P}}_{\mathcal S_l}\bm{\mathscr L}^*\bm{\mathscr L}
			(\bm{\mathscr I}-\bm{\mathscr{P}}_{\mathcal S_l})(\mathcal X_l-\mathcal T)
			\right\|_F
		}_{I_3}
		+
		\underbrace{
			\alpha_l
			\left\|
			(\bm{\mathscr{P}}_l^{\boldsymbol\omega_l}-\bm{\mathscr{P}}_{\mathcal S_l})
			\bm{\mathscr L}^*\bm{\mathscr L}(\mathcal X_l-\mathcal T)
			\right\|_F
		}_{I_4}
		\Big)
		\nonumber\\
		&=
		(\sqrt d+1)(I_1+I_2+I_3+I_4).
	\end{align*}
	We next bound these four terms separately.
	
	\begin{itemize}[leftmargin=*]
		\item \textbf{Bound of $I_1$:} Writing $\bm{\mathscr{P}}_{\mathcal S_l}-\alpha_l\bm{\mathscr{P}}_{\mathcal S_l}\bm{\mathscr L}^*\bm{\mathscr L} \bm{\mathscr{P}}_{\mathcal S_l} = (\bm{\mathscr{P}}_{\mathcal S_l}-\bm{\mathscr{P}}_{\mathcal S_l}\bm{\mathscr L}^*\bm{\mathscr L} \bm{\mathscr{P}}_{\mathcal S_l}) + (1-\alpha_l)\bm{\mathscr{P}}_{\mathcal S_l}\bm{\mathscr L}^*\bm{\mathscr L} \bm{\mathscr{P}}_{\mathcal S_l}$, Lemma~\ref{lem:app_restricted_trip_operator} and \eqref{eq:weighted_stepsize_deviation} yield
		\begin{equation*}
			I_1 \le \left( \delta_{2\boldsymbol r} + \frac{\delta_{2\boldsymbol r}+\rho_l-\delta_{2\boldsymbol r}\rho_l}{(1-\delta_{2\boldsymbol r})(1-\rho_l)}(1+\delta_{2\boldsymbol r}) \right) \|\mathcal X_l-\mathcal T\|_F = \frac{2\delta_{2\boldsymbol r}+\rho_l-\delta_{2\boldsymbol r}\rho_l}{(1-\delta_{2\boldsymbol r})(1-\rho_l)} \|\mathcal X_l-\mathcal T\|_F.
		\end{equation*}
		
		\item \textbf{Bound of $I_2$:} Applying Lemma~\ref{lem:app_tangent_mismatch} directly gives
		\begin{equation*}
			I_2 = \left\| (\bm{\mathscr I}-\bm{\mathscr{P}}_{\mathcal S_l})\mathcal T \right\|_F \le \frac{2^d-1}{\min_i\sigma_{r_i}(\mathcal T_{(i)})} \|\mathcal X_l-\mathcal T\|_F^2.
		\end{equation*}
		
		\item \textbf{Bound of $I_3$:} Lemmas~\ref{lem:app_normal_residual_cross_term} and \ref{lem:weighted_stepsize_bound} imply
		\begin{equation*}
			I_3 \le \frac{1+\delta_{2\boldsymbol r}}{(1-\delta_{2\boldsymbol r})(1-\rho_l)} \frac{2^d-1}{\min_i\sigma_{r_i}(\mathcal T_{(i)})} \|\mathcal X_l-\mathcal T\|_F^2.
		\end{equation*}
		
		\item \textbf{Bound of $I_4$:} Combining Lemmas~\ref{lem:weighted_cross_term} and \ref{lem:weighted_stepsize_bound} leads to
		\begin{equation*}
			I_4 \le \frac{\rho_l(1+\delta_{2\boldsymbol r})}{(1-\delta_{2\boldsymbol r})(1-\rho_l)} \left[ \|\mathcal X_l-\mathcal T\|_F + \frac{2^d-1}{\min_i\sigma_{r_i}(\mathcal T_{(i)})} \|\mathcal X_l-\mathcal T\|_F^2 \right].
		\end{equation*}
	\end{itemize}
	
	Combining bounds of $I_1$, $I_2$, $I_3$, and $I_4$ yields
	\begin{align}
		\|\mathcal X_{l+1}-\mathcal T\|_F
		&\le
		\frac{2(\sqrt d+1)}
		{(1-\delta_{2\boldsymbol r})(1-\rho_l)}
		\Bigg[
		\delta_{2\boldsymbol r}+\rho_l
		+
		\frac{(2^d-1)(1+\delta_{2\boldsymbol r}\rho_l)}
		{\min_i\sigma_{r_i}(\mathcal T_{(i)})}
		\|\mathcal X_l-\mathcal T\|_F
		\Bigg]
		\|\mathcal X_l-\mathcal T\|_F.
		\label{eq:weighted_one_step_recurrence_rhol}
	\end{align}
	
	We proceed by induction. For $l=0$, the assumption $\|\mathcal X_0-\mathcal T\|_F \le R$ and \eqref{eq:local_error_implies_rho} ensure $\rho_0 \le \rho_\star$. Assume $\|\mathcal X_l-\mathcal T\|_F \le R$, which implies $\rho_l \le \rho_\star$. The coefficient multiplying
	$\|\mathcal X_l-\mathcal T\|_F$ in
	\eqref{eq:weighted_one_step_recurrence_rhol}
	is monotonically increasing with respect to both
	$\|\mathcal X_l-\mathcal T\|_F$ and $\rho_l$. Bounding them by $R$ and $\rho_\star$ respectively gives:
	\begin{equation*}
		\|\mathcal X_{l+1}-\mathcal T\|_F \le \gamma_R \|\mathcal X_l-\mathcal T\|_F.
	\end{equation*}
	Since $\gamma_R < 1$, we obtain $\|\mathcal X_{l+1}-\mathcal T\|_F \le R$, completing the induction. Iterating this strict contraction yields $\|\mathcal X_l-\mathcal T\|_F \le \gamma_R^l \|\mathcal X_0-\mathcal T\|_F$, and $\rho_l \le \rho_\star$ holds uniformly for all $l \ge 0$.
\end{proof}

\subsection{Standard RGD initialization}
\label{app:standard_initialization}

The following result verifies the initialization condition used in Corollary~\ref{cor:weighted_standard_initialization} for the initialization prescribed in Algorithm~\ref{alg:abw_modewise_rgrad}.

\begin{lemma}[Initialization error]
	\label{lem:app_initialization_error}
	Let $\boldsymbol{y} = \bm{\mathscr L}(\mathcal{T})$ and $\mathcal{X}_0 = \bm{\mathscr H}_{\boldsymbol r}(\bm{\mathscr L}^*\boldsymbol{y})$. If $\bm{\mathscr L}$ satisfies $\operatorname{TRIP}(\delta_{2\boldsymbol r}, 2\boldsymbol{r})$, then \eqref{eq:weighted_initial_error} holds.
\end{lemma}

\begin{proof}
	For each $i\in[d]$, let
	$K^{(i)}\in\mathbb R^{n_i\times t_i}$, with $t_i\le2r_i$,
	have orthonormal columns spanning the sum of
	$\operatorname{span}(U^{(i)})$ and the subspace generated by the
	leading $r_i$ left singular vectors of
	$(\bm{\mathscr L}^*\boldsymbol y)_{(i)}$. Define the orthogonal projector $\bm{\mathscr P}_K := \prod_{i=1}^d \bm{\mathscr P}_{K^{(i)}}^{(i)}$. By construction, $\bm{\mathscr P}_K \mathcal{T} = \mathcal{T}$, and $\operatorname{mulrank}(\bm{\mathscr P}_K \mathcal{Z}) \preceq 2\boldsymbol{r}$ for any ambient tensor $\mathcal{Z}$.
	
	According to the ordering property of the HOSVD \cite{de2000multilinear}, $\bm{\mathscr P}_K\bm{\mathscr L}^*\boldsymbol{y}$ and $\bm{\mathscr L}^*\boldsymbol{y}$ share identical leading singular vectors and values, yielding $\mathcal{X}_0 = \bm{\mathscr H}_{\boldsymbol r}(\bm{\mathscr L}^*\boldsymbol{y}) = \bm{\mathscr H}_{\boldsymbol r}(\bm{\mathscr P}_K\bm{\mathscr L}^*\boldsymbol{y})$. The quasi-optimality of the truncated HOSVD implies
	\begin{equation*}
		\|\mathcal{X}_0 - \bm{\mathscr P}_K\bm{\mathscr L}^*\boldsymbol{y}\|_F \le \sqrt{d} \|\mathcal{T} - \bm{\mathscr P}_K\bm{\mathscr L}^*\boldsymbol{y}\|_F.
	\end{equation*}
	Applying the triangle inequality, $\boldsymbol{y} = \bm{\mathscr L}(\mathcal{T})$, and $\bm{\mathscr P}_K\mathcal{T} = \mathcal{T}$, we obtain
	\begin{align*}
		\|\mathcal{X}_0-\mathcal{T}\|_F
		&\le
		\|\mathcal{X}_0-\bm{\mathscr P}_K\bm{\mathscr L}^*\boldsymbol{y}\|_F
		+\|\bm{\mathscr P}_K\bm{\mathscr L}^*\boldsymbol{y}-\mathcal{T}\|_F
		\le
		(\sqrt{d}+1)\|\mathcal{T}-\bm{\mathscr P}_K\bm{\mathscr L}^*\bm{\mathscr L}\mathcal{T}\|_F
		\nonumber\\
		&=
		(\sqrt{d}+1)
		\|(\bm{\mathscr P}_K-\bm{\mathscr P}_K\bm{\mathscr L}^*\bm{\mathscr L}\bm{\mathscr P}_K)\mathcal{T}\|_F.
	\end{align*}
	Since $\bm{\mathscr P}_K-\bm{\mathscr P}_K\bm{\mathscr L}^*\bm{\mathscr L} \bm{\mathscr P}_K$ is self-adjoint and
	$\operatorname{mulrank}(\bm{\mathscr P}_K\mathcal Z)\preceq2\boldsymbol r$,
	the TRIP implies
	\begin{align*}
		\|\bm{\mathscr P}_K-\bm{\mathscr P}_K\bm{\mathscr L}^*\bm{\mathscr L} \bm{\mathscr P}_K\|
		&=
		\sup_{\|\mathcal Z\|_F=1}
		\left|
		\left\langle
		(\bm{\mathscr P}_K-\bm{\mathscr P}_K\bm{\mathscr L}^*\bm{\mathscr L} \bm{\mathscr P}_K)\mathcal Z,\mathcal Z
		\right\rangle_F
		\right|
		\nonumber\\
		&=
		\sup_{\|\mathcal Z\|_F=1}
		\left|
		\|\bm{\mathscr P}_K\mathcal Z\|_F^2-\|\bm{\mathscr L}(\bm{\mathscr P}_K\mathcal Z)\|_2^2
		\right|
		\le\delta_{2\boldsymbol r}.
	\end{align*}
	Consequently,
	\begin{equation*}
		\|\mathcal X_0-\mathcal T\|_F
		\le
		(\sqrt d+1)\delta_{2\boldsymbol r}\|\mathcal T\|_F.
	\end{equation*}
\end{proof}

\bibliographystyle{unsrt}
\bibliography{references}
      	
\end{document}